%% file: paper.tex
\documentclass{article}
\usepackage{iclr2027_conference,times}
\usepackage{amsmath,amssymb,amsthm,booktabs,multirow,array}
\usepackage[table]{xcolor}
\definecolor{mergeheirrow}{RGB}{245,240,250}
\usepackage{graphicx}
\usepackage{hyperref}
\hypersetup{hidelinks}
\usepackage{url}
\usepackage{float}
\usepackage{subcaption}

\newtheorem{theorem}{Theorem}[section]
\newtheorem{proposition}[theorem]{Proposition}

\newtheorem{corollary}[theorem]{Corollary}

\newcolumntype{C}[1]{>{\centering\arraybackslash}p{#1}}

\title{MergeHEIR: Mitigating Multimodal Hallucinations as the Tax of Model Merging}
\author{\normalfont
Jinyu Li\textsuperscript{1,*} \quad
Hao Fang\textsuperscript{2,*} \quad
Zhiming Zhang\textsuperscript{1} \quad
Jiawei Kong\textsuperscript{2} \quad
Bin Chen\textsuperscript{1} \quad
Shu-Tao Xia\textsuperscript{2} \\
\textsuperscript{1}Harbin Institute of Technology, Shenzhen \\
\textsuperscript{2}Tsinghua Shenzhen International Graduate School, Tsinghua University \\
\textsuperscript{*}Equal contribution
}

\iclrfinalcopy

\begin{document}
\maketitle
\lhead{}

\input{secs/0_abstract}
\input{secs/1_introduction}
\input{secs/3_problem_setup}
\input{secs/4_method}
\input{secs/5_experiments}
\input{secs/6_conclusion}

\bibliographystyle{iclr2027_conference}
\bibliography{references}

\clearpage
\appendix
\input{secs/2_related_work}
\input{secs/A_merging_tax_analysis}
\input{secs/A_theoretical_analysis}
\input{secs/A_experimental_details}
\input{secs/A_additional_results}

\clearpage
\input{secs/A_case_study}
\input{secs/X_statements}
\input{secs/A_llm_use}

\end{document}

%% file: secs/0_abstract.tex

\begin{abstract}
Model merging consolidates task-specialized experts into a single deployable model. However, we show that such capability consolidation incurs a merging tax of increased hallucination: across 8 model-merging methods, every merged checkpoint exhibits a higher hallucination rate than the average of its constituent experts. An intuitive approach is to adapt existing hallucination-mitigation methods to the post-merge model, yet this unconstrained adaptation disrupts inherited capabilities, creating a tension between hallucination mitigation and expertise retention. 
To tackle this challenge, we introduce MergeHEIR, a post-merge adaptation framework designed to reduce this merging tax while preserving expertise inherited from initial experts. Using small expert-task calibration sets, MergeHEIR constructs layer-wise null-space projectors via SVD from task-specific activations collected from the merged checkpoint, and periodically projects the accumulated post-merge displacement onto the resulting null spaces to preserve inherited expertise.
Theoretically, we establish minimum-distortion and maximum-dimensionality guarantees, characterize the threshold-controlled adaptation-retention trade-off, and extend perturbation guarantees beyond finite calibration data. Across 24 paired comparisons spanning three MLLM configurations and 8 model-merging methods, MergeHEIR consistently mitigates hallucination while largely preserving inherited expertise, demonstrating a more favorable hallucination-retention trade-off.
\end{abstract}

%% file: secs/1_introduction.tex
\section{Introduction}

Model merging consolidates independently specialized models into a single deployable checkpoint, providing a practical way to combine complementary capabilities without jointly retraining a unified model on all constituent data \citep{wortsman2022model,ilharco2023editing}. This paradigm is particularly appealing for multimodal large language models (MLLMs), for which capabilities such as chart understanding, geometry, optical character recognition, and visual question answering are often distributed across specialized checkpoints \citep{wei2025optmerge}. However, capability consolidation alone does not guarantee deployment reliability.

We first investigate visual hallucination after model merging and reveal a systematic \emph{merging tax}: across eight merging methods and two Qwen2.5-VL \citep{bai2025qwen2} scales, every evaluated merged checkpoint exhibits a higher AMBER hallucination rate than the average of its constituent experts, as shown in Figure~\ref{fig:preliminary-merge-hallucination}. The mean rate increases from $33.48$ to $54.35$ for Qwen2.5-VL-3B and from $20.83$ to $25.41$ for Qwen2.5-VL-7B. Notably, all constituent experts share the same frozen vision encoder and differ only through language-side LoRA specialization. We attribute the resulting tax to \emph{visual-evidence utilization collision}: expert specialization forms co-adapted language-side computations for using visual evidence, while model merging can perturb the cross-layer interactions through which this evidence influences generation. As a result, the encoded visual representation may remain intact while its downstream use becomes less effective, increasing susceptibility to visually unsupported generation. Appendix~\ref{app:merging-tax-mechanism} provides a mechanistic interpretation of this effect. These results show that consolidating complementary capabilities does not guarantee reliable visual grounding, making hallucination correction a necessary post-merge stage.

\begin{figure*}[t]
\centering
\includegraphics[width=0.98\textwidth]{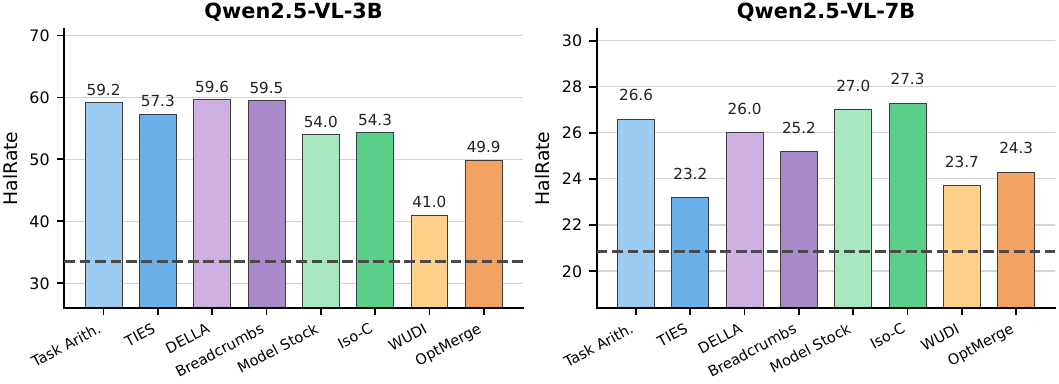}
\caption{Merging tax manifested as elevated hallucination rates. We evaluate eight merging methods with AMBER \citep{wang2023amber} on Qwen2.5-VL-3B and Qwen2.5-VL-7B. Dashed lines indicate the corresponding expert-average baselines, with HalRates of 33.5 and 20.8, respectively. Every evaluated merged checkpoint lies above its corresponding baseline.}
\label{fig:preliminary-merge-hallucination}
\vspace{-0.5em}
\end{figure*}

\begin{figure*}[t]
\centering
\includegraphics[width=0.96\textwidth]{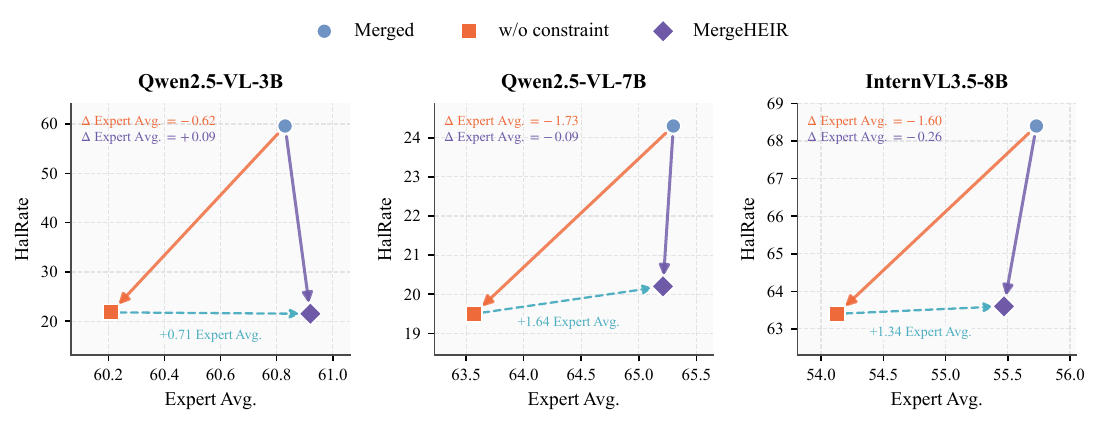}
\caption{Matched comparison of post-merge adaptation outcomes across three model configurations. HalRate denotes the AMBER hallucination rate, while Expert Avg. is the mean performance across the seven expert-task benchmarks. Unconstrained adaptation achieves comparable hallucination mitigation but incurs a substantially larger reduction in Expert Avg., whereas MergeHEIR better preserves inherited expertise while achieving effective hallucination mitigation.}
\label{fig:retention-tradeoff}
\vspace{-0.5em}
\end{figure*}

A natural remedy is to adapt the merged model through hallucination-oriented multimodal preference optimization \citep{wang2024mdpo}. Yet this correction introduces another problem. The same updates used to strengthen visual grounding can inadvertently disrupt the heterogeneous capabilities inherited from the constituent experts. Unconstrained adaptation that optimizes hallucination alone can therefore degrade expert-task performance. Figure~\ref{fig:retention-tradeoff} visualizes this tension across the three model configurations: unconstrained adaptation achieves hallucination mitigation comparable to MergeHEIR but consistently incurs a larger reduction in Expert Avg. This degradation reflects adaptation-induced interference with the expertise inherited from the constituent models. The resulting challenge is not merely to reduce hallucination, but to do so without undoing the capabilities that merging was intended to consolidate. Post-merge hallucination mitigation therefore cannot be treated as behavioral correction alone; it must improve visual grounding while limiting degradation of the inherited expert capabilities. This leads to our central question: \emph{how can hallucination-oriented preference optimization mitigate post-merge hallucination while limiting interference with the expert capabilities consolidated through model merging?}

To address this question, we propose MergeHEIR, a post-merge adaptation framework that coordinates hallucination correction with the retention of inherited expertise. It combines hallucination-oriented preference optimization with activation-derived null-space projection to suppress adaptation-induced interference with inherited expertise. Using only a small calibration set, MergeHEIR constructs null spaces from the merged model's activation statistics and periodically projects accumulated post-merge updates onto the corresponding null spaces to preserve expert-task capabilities. We further show that this is a minimum-distortion projection onto a maximum-dimensional admissible subspace, derive pooled and task-specific representation perturbation bounds, characterize the threshold-controlled adaptation--retention trade-off, and extend the guarantee to the corresponding activation distribution.

We evaluate MergeHEIR through 24 paired comparisons spanning three MLLM configurations and eight merging methods. Across these settings, MergeHEIR consistently mitigates hallucination while largely preserving the expert capabilities consolidated through model merging. Matched comparisons against unconstrained post-merge adaptation further show that null-space projection consistently improves expert retention while maintaining effective hallucination mitigation, yielding a more favorable hallucination--retention trade-off.

\noindent{Our contributions are threefold:}
\begin{itemize}
\item We identify a two-stage reliability problem in multimodal model merging: capability consolidation incurs systematic \emph{merging tax} from disrupted language-side visual-evidence utilization, while unconstrained hallucination correction can interfere with inherited expertise.
\item We propose MergeHEIR, a capability-preserving framework that constructs SVD-derived activation-insensitive subspaces and periodically projects accumulated updates onto them. We also establish theoretical guarantees on adaptation geometry and perturbation control.
\item We validate the effectiveness and robustness of MergeHEIR through 24 paired comparisons across three MLLM configurations and eight merging methods, complemented by controlled comparisons and threshold analyses that isolate the role of the constraint.
\end{itemize}

%% file: secs/3_problem_setup.tex
\section{Problem Formulation}
\label{sec:problem}

Figure~\ref{fig:preliminary-merge-hallucination} reveals a systematic \emph{merging tax}: across the evaluated merging methods, the resulting checkpoints exhibit consistently higher hallucination rates than their corresponding constituent-expert averages. This motivates hallucination-oriented post-merge adaptation. However, such adaptation modifies parameters that also carry the consolidated expert capabilities, creating a conflict between hallucination correction and expert retention. Post-merge adaptation thus poses a two-objective problem: improving visual grounding while retaining expert-task performance. We formulate this problem as hallucination correction under expert-retention constraints.

\subsection{Merged model and retention targets}
Let $\mathcal{T}=\{\mathrm{Chart},\mathrm{Geometry},\mathrm{OCR},\mathrm{VQA}\}$ denote the set of expert tasks, and let $\theta_t$ denote the parameters of the expert model specialized for task $t\in\mathcal{T}$. A model-merging operation combines these expert parameters into a single checkpoint:
\begin{equation}
\theta_m=\operatorname{Merge}\!\left(\{\theta_t\}_{t\in\mathcal{T}}\right).
\end{equation}
Post-merge adaptation initializes the trainable parameters $\theta$ from $\theta_m$, while $\theta_m$ remains fixed as the reference checkpoint. For each task $t$, let $\mathcal{D}_t$ denote a representative expert-task distribution. The collection $\{\mathcal{D}_t\}_{t\in\mathcal{T}}$ captures the heterogeneous expert capabilities that should remain stable.

\subsection{Hallucination correction under expert-retention constraints}

Let $\mathcal{L}_{\mathrm{hall}}(\theta)$ denote a hallucination-oriented adaptation objective applied to the merged checkpoint. Optimizing this objective alone imposes no restriction on behavioral changes over the expert-task distributions. We therefore formulate post-merge adaptation as
\begin{equation}
\min_{\theta}\;\mathcal{L}_{\mathrm{hall}}(\theta)
\quad
\text{s.t.}\quad
\mathcal{R}_t(\theta,\theta_m)\leq\epsilon_t,
\qquad \forall t\in\mathcal{T},
\end{equation}
where $\mathcal{R}_t(\theta,\theta_m)$ measures the adaptation-induced functional deviation from the original merged checkpoint on $\mathcal{D}_t$, and $\epsilon_t$ denotes the tolerated deviation. MergeHEIR instantiates these retention constraints through activation-derived null-space projection, as described in the next section.

%% file: secs/4_method.tex
\section{Method}
\label{sec:method}

MergeHEIR combines hallucination-oriented multimodal preference optimization with activation-derived null-space projection. The preference objective drives hallucination mitigation, while the null-space constraint restricts accumulated parameter changes to directions that produce limited perturbation on the expert-task calibration data. Figure~\ref{fig:overview} summarizes the complete framework.

\begin{figure*}[t]
\centering
\includegraphics[width=0.98\textwidth]{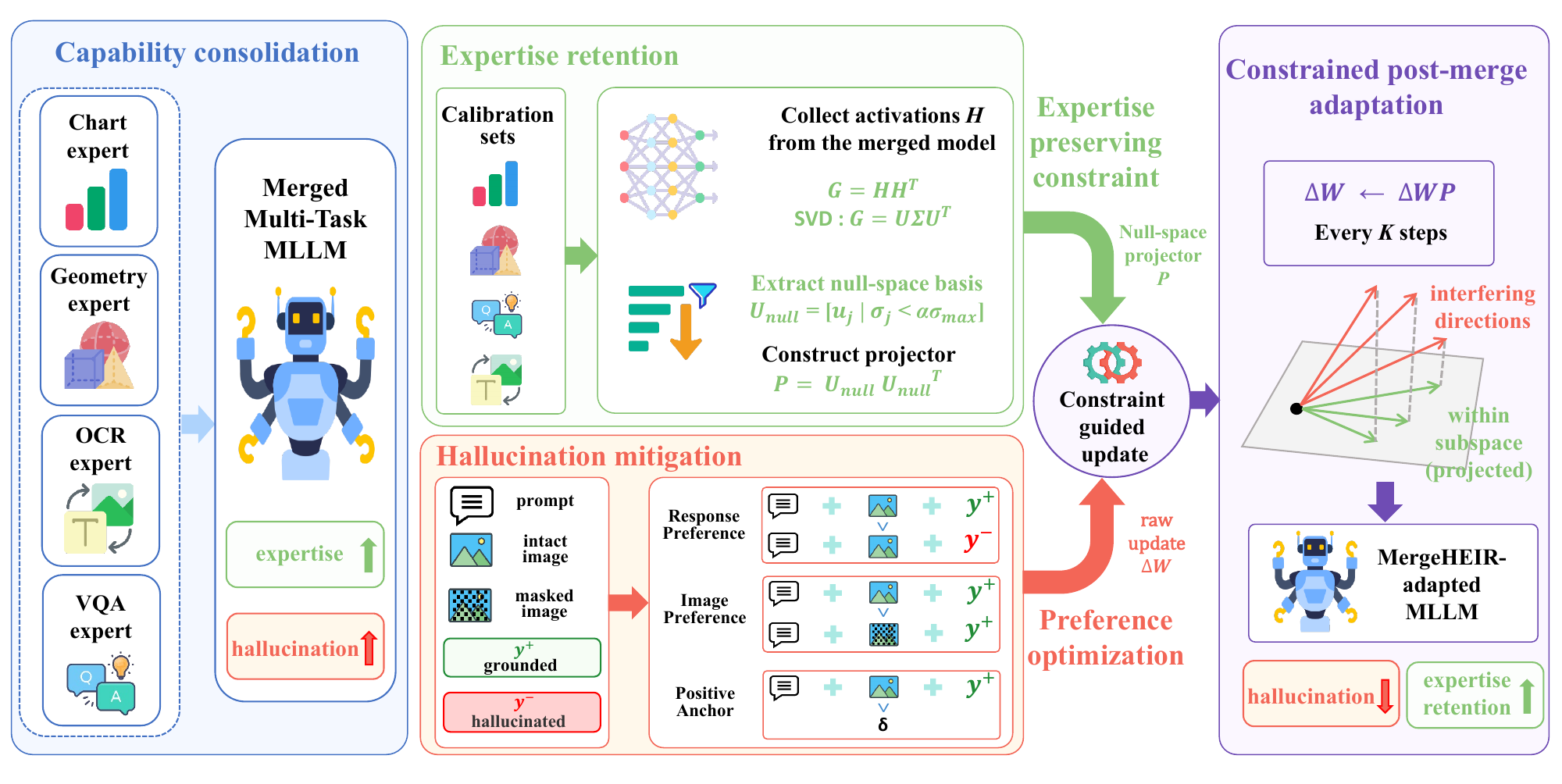}
\caption{Overview of MergeHEIR. Four task-specialized experts are first consolidated into a merged checkpoint $\theta_m$. Hallucination-oriented preference optimization produces corrective updates, while expert-task calibration activations define layer-wise null-space projectors. MergeHEIR periodically projects the accumulated parameter displacement onto the corresponding null spaces, reducing hallucination while retaining inherited expertise.}
\label{fig:overview}
\end{figure*}

\subsection{Hallucination-oriented preference optimization}
\label{sec:preference-objective}

Let $\pi_\theta$ denote the trainable policy initialized from the merged checkpoint $\theta_m$, and let $\pi_m$ denote the frozen reference policy. Each preference sample is given by $(x,I,y^+,y^-)$, where $y^+$ is grounded and $y^-$ is hallucinated. Following mDPO \citep{wang2024mdpo}, we define the reference-relative score as
$s_\theta(y\mid x,I)=\beta\log\frac{\pi_\theta(y\mid x,I)}{\pi_m(y\mid x,I)}$.
To strengthen image-conditioned preference, we further construct a perturbed image $\widetilde I$ via random region masking following OPA-DPO \citep{yang2025mitigating}. The preference objective consists of
\begin{align}
\mathcal{L}_{\mathrm{DPO}} &= -\mathbb{E}\log\sigma\!\left(s_\theta(y^+\mid x,I)-s_\theta(y^-\mid x,I)\right), \nonumber\\
\mathcal{L}_{\mathrm{img}} &= -\mathbb{E}\log\sigma\!\left(s_\theta(y^+\mid x,I)-s_\theta(y^+\mid x,\widetilde I)\right), \nonumber\\
\mathcal{L}_{\mathrm{anc}} &= -\mathbb{E}\log\sigma\!\left(s_\theta(y^+\mid x,I)-\delta\right).
\end{align}
Here, $\mathcal{L}_{\mathrm{DPO}}$ favors grounded responses over hallucinated ones, $\mathcal{L}_{\mathrm{img}}$ strengthens reliance on the original visual evidence, and $\mathcal{L}_{\mathrm{anc}}$ prevents the preference margin from increasing solely through suppression of the rejected response. The complete objective is
\begin{equation}
\mathcal{L}_{\mathrm{pref}}
=
\mathcal{L}_{\mathrm{DPO}}
+\lambda_{\mathrm{img}}\mathcal{L}_{\mathrm{img}}
+\lambda_{\mathrm{anc}}\mathcal{L}_{\mathrm{anc}}.
\label{eq:loss}
\end{equation}

\subsection{Expert-task null-space projection}
\label{sec:null-space}

The hallucination-oriented objective specifies the desired behavioral correction, but does not explicitly constrain its impact on the expert capabilities consolidated in the merged model. To address this, we propose to appropriately regulate the model update during training to reduce interference with the learned expertise. 
Specifically, MergeHEIR constructs layer-wise null spaces from expert-task activations and constrains post-merge model updates for hallucination correction to these spaces, to preserve the acquired expert capabilities.

\paragraph{Preservation condition.}
Consider a linear layer with weight matrix $W_m$ in the original merged checkpoint and a post-merge parameter update $\Delta W$. For an input activation $h$, the induced change in the layer output is
\begin{equation}
(W_m+\Delta W)h-W_mh
=
\Delta Wh.
\end{equation}
For each constituent expert task $t\in\mathcal{T}$, we form a small calibration set by randomly sampling task examples. We feed these examples through the merged model and collect the resulting input activations into a matrix $H_t$. We then concatenate the task-specific matrices along the sample dimension:
\begin{equation}
H=
\left[
H_1,H_2,\ldots,H_{|\mathcal{T}|}
\right].
\end{equation}
The joint feature matrix $H$ characterizes the activation subspace induced by the expert-task calibration data. An update satisfying
\begin{equation}
\Delta WH=0
\end{equation}
leaves the original layer outputs unchanged on all collected activations. This condition motivates restricting post-merge updates to directions orthogonal to the activation subspace represented by the joint matrix $H$ consisting of different task features.

\paragraph{Null-space construction.}
Directly extracting a left null-space basis from the activation matrix $H$ becomes costly when a large number of token activations are collected. Since $H$ and $HH^\top$ share the same left null space, we instead operate on the smaller activation Gram matrix and construct the projector via singular value decomposition (SVD):
\begin{equation}
G=HH^\top=U\Sigma U^\top,
\qquad
U_{\mathrm{null}}
=
\left[
u_j
\mid
\sigma_j<\alpha\sigma_{\max}
\right],
\qquad
P=U_{\mathrm{null}}U_{\mathrm{null}}^\top.
\label{eq:expert-geometry}
\end{equation}
Here, $\alpha$ defines an approximate null space by selecting singular vectors satisfying $\sigma_j/\sigma_{\max}<\alpha$. We consider $\alpha\in[5\times10^{-6},5\times10^{-2}]$ and use $\alpha=5\times10^{-4}$ by default to balance adaptation flexibility and expert-capability preservation. MergeHEIR then projects the post-merge update as $\Delta W\leftarrow\Delta WP$, suppressing components aligned with activation-sensitive directions and confining adaptation to directions weakly represented in the expert-task activation space.

\subsection{Geometric Optimality and Perturbation Control}
\label{sec:geometric-analysis}

The projection in MergeHEIR must preserve useful hallucination-correction updates while limiting changes to expert-task representations. We first establish the minimum-distortion property within a fixed subspace, then examine spectral perturbation control and threshold-dependent adaptation flexibility, and finally extend this control beyond calibration data. Let $P$ be any fixed orthogonal projection matrix, and let $\mathcal{S}_{P}=\{Z:Z=ZP\}$ denote its induced admissible displacement subspace.

\begin{proposition}[Minimum-distortion projection]
\label{prop:min-distortion}
For any accumulated displacement $\Delta W$,
\begin{equation}
\Delta WP=\operatorname*{arg\,min}_{Z\in\mathcal{S}_{P}}\|Z-\Delta W\|_{F}^{2}.
\end{equation}
Moreover, $\Delta W=\Delta WP+\Delta W(I-P)$ is an orthogonal decomposition, so $\|\Delta WP\|_{F}\leq\|\Delta W\|_{F}$.
\end{proposition}

Proposition~\ref{prop:min-distortion} establishes that null-space projection enforces the subspace constraint with minimum Frobenius-norm distortion. MergeHEIR therefore retains all admissible components of the preference update, avoiding unnecessary suppression of the learned hallucination-correction update.

\begin{theorem}[Spectral perturbation control]
\label{thm:spectral-control}
The range of the projector in Eq.~\eqref{eq:expert-geometry} is maximum-dimensional among subspaces satisfying $v^{\top}Gv<\alpha\sigma_{\max}\|v\|_{2}^{2}$ for every nonzero $v$. For any displacement $\Delta W$,
\begin{equation}
\|\Delta WPH_t\|_{F}^{2}
\leq
\|\Delta WPH\|_{F}^{2}
\leq
\alpha\sigma_{\max}\|\Delta WP\|_{F}^{2},
\qquad
\forall t\in\mathcal{T}.
\end{equation}
\end{theorem}

Theorem~\ref{thm:spectral-control} implies that choosing an appropriately small $\alpha$ allows MergeHEIR to limit layer-output perturbations on expert-task calibration activations. This reduces interference with inherited task representations during hallucination correction. Meanwhile, the maximum-dimensional construction simultaneously retains as many admissible update directions as the sensitivity constraint permits.

\begin{table*}[t]
\caption{Expert-task results on Qwen2.5-VL models.}
\vspace{-1pt}
\label{tab:qwen-expert}
\centering
\scriptsize
\setlength{\tabcolsep}{1.3pt}

\begin{subtable}{\textwidth}
\centering
\begin{tabular}{@{}C{70pt}C{38pt}|C{27pt}|cc|C{27pt}C{27pt}|C{27pt}C{27pt}|C{27pt}@{}}
\toprule
\multirow[c]{2}{*}{Methods} & \multirow[c]{2}{*}{Stage} & \multicolumn{1}{c|}{Chart} & \multicolumn{2}{c|}{Geometry} & \multicolumn{2}{c|}{OCR} & \multicolumn{2}{c|}{VQA} & \multirow[c]{2}{*}{Avg.}\\
 & & ChartQA & MATH-Vision & Geometry3K & TextVQA & DocVQA & GQA & OKVQA\\
\midrule
Task Arithmetic & Merged & 83.96 & 20.72 & 48.25 & 78.67 & 92.66 & 60.98 & 40.48 & 60.82\\
\rowcolor{mergeheirrow} & MergeHEIR & 84.20 & 21.38 & 48.75 & 78.72 & 92.66 & 60.90 & 40.93 & 61.08\\
TIES Merging & Merged & 83.92 & 21.71 & 48.75 & 78.53 & 92.54 & 61.15 & 40.18 & 60.97\\
\rowcolor{mergeheirrow} & MergeHEIR & 83.96 & 20.72 & 47.92 & 78.67 & 92.67 & 60.96 & 40.09 & 60.71\\
DELLA Merging & Merged & 83.68 & 21.05 & 48.75 & 78.55 & 92.57 & 60.86 & 40.35 & 60.83\\
\rowcolor{mergeheirrow} & MergeHEIR & 84.00 & 21.05 & 48.42 & 78.75 & 92.69 & 60.70 & 40.84 & 60.92\\
Breadcrumbs Merging & Merged & 83.96 & 20.39 & 48.42 & 78.87 & 92.59 & 61.18 & 40.34 & 60.82\\
\rowcolor{mergeheirrow} & MergeHEIR & 84.08 & 21.05 & 48.09 & 78.80 & 92.71 & 61.00 & 40.63 & 60.91\\
Model Stock & Merged & 83.56 & 20.72 & 48.09 & 78.15 & 92.44 & 60.35 & 41.13 & 60.63\\
\rowcolor{mergeheirrow} & MergeHEIR & 83.48 & 21.05 & 47.75 & 78.20 & 92.39 & 59.87 & 41.70 & 60.64\\
Iso-C & Merged & 83.44 & 21.05 & 48.42 & 78.14 & 92.48 & 60.35 & 41.04 & 60.70\\
\rowcolor{mergeheirrow} & MergeHEIR & 83.76 & 21.38 & 47.25 & 78.38 & 92.46 & 60.37 & 41.36 & 60.71\\
WUDI Merging & Merged & 84.08 & 21.05 & 48.59 & 79.24 & 92.75 & 62.28 & 39.41 & 61.06\\
\rowcolor{mergeheirrow} & MergeHEIR & 84.36 & 20.07 & 48.25 & 79.56 & 93.23 & 62.23 & 39.58 & 61.04\\
OptMerge & Merged & 84.08 & 21.71 & 48.09 & 78.85 & 92.55 & 61.70 & 40.19 & 61.02\\
\rowcolor{mergeheirrow} & MergeHEIR & 84.48 & 21.71 & 47.59 & 79.05 & 92.62 & 61.28 & 40.69 & 61.06\\
\bottomrule
\end{tabular}
\caption{Qwen2.5-VL-3B}
\end{subtable}

\vspace{5pt}

\begin{subtable}{\textwidth}
\centering
\begin{tabular}{@{}C{70pt}C{38pt}|C{27pt}|cc|C{27pt}C{27pt}|C{27pt}C{27pt}|C{27pt}@{}}
\toprule
\multirow[c]{2}{*}{Methods} & \multirow[c]{2}{*}{Stage} & \multicolumn{1}{c|}{Chart} & \multicolumn{2}{c|}{Geometry} & \multicolumn{2}{c|}{OCR} & \multicolumn{2}{c|}{VQA} & \multirow[c]{2}{*}{Avg.}\\
 & & ChartQA & MATH-Vision & Geometry3K & TextVQA & DocVQA & GQA & OKVQA\\
\midrule
Task Arithmetic & Merged & 86.32 & 26.32 & 54.41 & 83.26 & 94.62 & 60.94 & 44.42 & 64.32\\
\rowcolor{mergeheirrow} & MergeHEIR & 85.76 & 26.32 & 53.58 & 82.71 & 94.76 & 60.90 & 44.21 & 64.03\\
TIES Merging & Merged & 84.12 & 25.00 & 53.41 & 82.87 & 94.73 & 60.60 & 43.79 & 63.50\\
\rowcolor{mergeheirrow} & MergeHEIR & 82.52 & 23.36 & 52.91 & 82.21 & 94.66 & 60.59 & 43.04 & 62.76\\
DELLA Merging & Merged & 87.24 & 25.99 & 54.91 & 83.46 & 94.50 & 61.52 & 45.44 & 64.72\\
\rowcolor{mergeheirrow} & MergeHEIR & 87.16 & 24.34 & 55.41 & 83.84 & 94.82 & 61.54 & 44.81 & 64.56\\
Breadcrumbs Merging & Merged & 87.24 & 26.32 & 54.58 & 83.45 & 94.44 & 61.40 & 45.40 & 64.69\\
\rowcolor{mergeheirrow} & MergeHEIR & 87.16 & 24.67 & 55.07 & 83.96 & 94.80 & 61.62 & 44.64 & 64.56\\
Model Stock & Merged & 83.72 & 23.03 & 52.41 & 82.84 & 94.71 & 60.54 & 43.45 & 62.96\\
\rowcolor{mergeheirrow} & MergeHEIR & 81.56 & 23.03 & 52.25 & 81.95 & 94.70 & 60.41 & 43.30 & 62.46\\
Iso-C & Merged & 83.40 & 23.03 & 50.75 & 82.87 & 94.68 & 60.52 & 43.52 & 62.68\\
\rowcolor{mergeheirrow} & MergeHEIR & 82.00 & 22.70 & 52.25 & 82.03 & 94.72 & 60.52 & 43.04 & 62.47\\
WUDI Merging & Merged & 87.52 & 27.30 & 54.08 & 83.12 & 94.41 & 63.40 & 48.74 & 65.51\\
\rowcolor{mergeheirrow} & MergeHEIR & 86.52 & 26.97 & 55.41 & 82.33 & 94.51 & 63.48 & 48.39 & 65.37\\
OptMerge & Merged & 87.44 & 27.30 & 54.24 & 83.18 & 94.41 & 62.55 & 47.99 & 65.30\\
\rowcolor{mergeheirrow} & MergeHEIR & 87.00 & 27.63 & 53.74 & 82.86 & 94.74 & 62.73 & 47.75 & 65.21\\
\bottomrule
\end{tabular}
\caption{Qwen2.5-VL-7B}
\end{subtable}

\vspace{-2.4em}
\end{table*}

\begin{table*}[t]
\caption{Hallucination-task results on Qwen2.5-VL models.}
\vspace{-1pt}
\label{tab:qwen-hallucination}
\centering
\scriptsize
\setlength{\tabcolsep}{1.9pt}

\begin{subtable}{\textwidth}
\centering
\begin{tabular}{@{}C{70pt}C{38pt}|ccc|cc|C{26pt}C{26pt}|c@{}}
\toprule
\multirow[c]{2}{*}{Methods} & \multirow[c]{2}{*}{Stage} & \multicolumn{3}{c|}{AMBER} & \multicolumn{2}{c|}{Object Hal} & \multicolumn{2}{c|}{POPE Adversarial} & \multicolumn{1}{c}{GPT-4 Eval}\\
 & & CHAIR$\downarrow$ & HalRate$\downarrow$ & Cog$\downarrow$ & CHAIR$_s\downarrow$ & CHAIR$_i\downarrow$ & Acc.$\uparrow$ & F1$\uparrow$ & SHR$\downarrow$\\
\midrule
Task Arithmetic & Merged & 8.5 & 59.2 & 7.0 & 31.7 & 8.5 & 82.2 & 83.5 & 28.3\\
\rowcolor{mergeheirrow} & MergeHEIR & 4.7 & 20.9 & 1.8 & 23.3 & 6.1 & 85.8 & 86.0 & 25.4\\
TIES Merging & Merged & 8.5 & 57.3 & 6.4 & 34.0 & 9.0 & 82.2 & 83.5 & 28.3\\
\rowcolor{mergeheirrow} & MergeHEIR & 5.1 & 23.0 & 2.0 & 21.7 & 5.9 & 85.8 & 86.1 & 26.5\\
DELLA Merging & Merged & 8.8 & 59.6 & 7.1 & 36.0 & 9.0 & 82.1 & 83.4 & 29.9\\
\rowcolor{mergeheirrow} & MergeHEIR & 4.8 & 21.5 & 1.9 & 21.7 & 6.4 & 85.8 & 86.0 & 27.3\\
Breadcrumbs Merging & Merged & 8.7 & 59.5 & 7.0 & 34.3 & 9.2 & 81.9 & 83.2 & 28.5\\
\rowcolor{mergeheirrow} & MergeHEIR & 5.0 & 22.4 & 1.9 & 22.0 & 6.7 & 85.7 & 86.0 & 25.2\\
Model Stock & Merged & 8.4 & 54.0 & 6.1 & 34.3 & 8.6 & 82.9 & 83.9 & 27.8\\
\rowcolor{mergeheirrow} & MergeHEIR & 4.9 & 20.4 & 1.8 & 22.7 & 6.0 & 86.6 & 86.5 & 25.6\\
Iso-C & Merged & 8.4 & 54.3 & 6.1 & 36.3 & 8.6 & 82.8 & 83.9 & 28.4\\
\rowcolor{mergeheirrow} & MergeHEIR & 5.4 & 22.6 & 2.1 & 21.3 & 5.9 & 85.7 & 86.0 & 26.0\\
WUDI Merging & Merged & 8.0 & 41.0 & 4.5 & 40.3 & 9.9 & 85.4 & 85.2 & 31.0\\
\rowcolor{mergeheirrow} & MergeHEIR & 4.9 & 21.2 & 2.0 & 23.7 & 5.6 & 89.1 & 87.9 & 27.2\\
OptMerge & Merged & 8.3 & 49.9 & 5.7 & 37.3 & 10.6 & 83.5 & 84.2 & 28.5\\
\rowcolor{mergeheirrow} & MergeHEIR & 5.2 & 25.2 & 2.5 & 22.0 & 6.9 & 87.3 & 87.0 & 25.1\\
\bottomrule
\end{tabular}
\caption{Qwen2.5-VL-3B}
\end{subtable}

\vspace{5pt}

\begin{subtable}{\textwidth}
\centering
\begin{tabular}{@{}C{70pt}C{38pt}|ccc|cc|C{26pt}C{26pt}|c@{}}
\toprule
\multirow[c]{2}{*}{Methods} & \multirow[c]{2}{*}{Stage} & \multicolumn{3}{c|}{AMBER} & \multicolumn{2}{c|}{Object Hal} & \multicolumn{2}{c|}{POPE Adversarial} & \multicolumn{1}{c}{GPT-4 Eval}\\
 & & CHAIR$\downarrow$ & HalRate$\downarrow$ & Cog$\downarrow$ & CHAIR$_s\downarrow$ & CHAIR$_i\downarrow$ & Acc.$\uparrow$ & F1$\uparrow$ & SHR$\downarrow$\\
\midrule
Task Arithmetic & Merged & 4.9 & 26.6 & 1.6 & 39.3 & 8.4 & 83.4 & 80.7 & 28.8\\
\rowcolor{mergeheirrow} & MergeHEIR & 4.0 & 21.5 & 1.1 & 23.3 & 6.2 & 87.0 & 86.0 & 25.9\\
TIES Merging & Merged & 4.8 & 23.2 & 1.1 & 39.3 & 8.3 & 85.5 & 84.4 & 29.7\\
\rowcolor{mergeheirrow} & MergeHEIR & 4.1 & 20.3 & 0.9 & 24.0 & 5.9 & 89.2 & 88.0 & 27.5\\
DELLA Merging & Merged & 4.9 & 26.0 & 1.5 & 42.3 & 9.2 & 84.0 & 81.8 & 28.5\\
\rowcolor{mergeheirrow} & MergeHEIR & 4.3 & 21.4 & 1.0 & 24.7 & 6.7 & 87.0 & 86.5 & 26.1\\
Breadcrumbs Merging & Merged & 4.7 & 25.2 & 1.6 & 41.0 & 8.8 & 83.9 & 81.7 & 29.4\\
\rowcolor{mergeheirrow} & MergeHEIR & 4.2 & 21.6 & 1.2 & 24.3 & 6.4 & 87.5 & 86.9 & 25.8\\
Model Stock & Merged & 5.0 & 27.0 & 1.6 & 42.7 & 8.9 & 83.2 & 80.4 & 28.3\\
\rowcolor{mergeheirrow} & MergeHEIR & 4.2 & 22.3 & 1.2 & 25.0 & 7.1 & 86.8 & 85.7 & 26.0\\
Iso-C & Merged & 5.2 & 27.3 & 1.7 & 39.7 & 8.8 & 83.4 & 80.5 & 28.7\\
\rowcolor{mergeheirrow} & MergeHEIR & 3.7 & 20.8 & 1.0 & 25.3 & 6.9 & 86.7 & 85.5 & 26.2\\
WUDI Merging & Merged & 4.9 & 23.7 & 1.3 & 37.3 & 8.7 & 86.0 & 85.1 & 29.8\\
\rowcolor{mergeheirrow} & MergeHEIR & 3.9 & 20.4 & 1.1 & 22.7 & 6.6 & 89.4 & 88.9 & 26.9\\
OptMerge & Merged & 4.9 & 24.3 & 1.4 & 41.0 & 9.0 & 86.1 & 84.9 & 30.1\\
\rowcolor{mergeheirrow} & MergeHEIR & 4.2 & 20.2 & 1.2 & 25.0 & 6.2 & 89.7 & 89.3 & 26.8\\
\bottomrule
\end{tabular}
\caption{Qwen2.5-VL-7B}
\end{subtable}

\vspace{-2.0em}
\end{table*}

\begin{proposition}[Threshold-controlled adaptation geometry]
\label{prop:threshold-geometry}
For $0\leq\alpha_1\leq\alpha_2$, let $P_1=P_{\alpha_1}$, $P_2=P_{\alpha_2}$, and $Q=P_2-P_1$. Then $\operatorname{range}(P_1)\subseteq\operatorname{range}(P_2)$ and $\operatorname{rank}(P_1)\leq\operatorname{rank}(P_2)$. Moreover, for any accumulated displacement $\Delta W$,
\begin{equation}
\|\Delta WP_2\|_F^2=\|\Delta WP_1\|_F^2+\|\Delta WQ\|_F^2,\qquad \|\Delta WP_2H\|_F^2=\|\Delta WP_1H\|_F^2+\|\Delta WQH\|_F^2.
\end{equation}
\end{proposition}

Proposition~\ref{prop:threshold-geometry} shows that tuning $\alpha$ allows MergeHEIR to balance preference-update retention against calibration perturbation. An appropriate threshold thus avoids overly restricting hallucination-correction updates while limiting perturbations to expert-task representations. We evaluate this adaptation--retention balance in Figure~\ref{fig:threshold-sensitivity}.

\begin{corollary}[Population extension]
\label{cor:population-extension}
Let $N$ be the number of columns in $H$, let $\widehat{G}=\frac{1}{N}HH^{\top}$ with $\widehat{\sigma}_{\max}=\lambda_{\max}(\widehat{G})$, and let $\Gamma=\mathbb{E}[hh^{\top}]$. If $\|\widehat{G}-\Gamma\|_{2}\leq\varepsilon$, then
\begin{equation}
\mathbb{E}\!\left[\|\Delta WPh\|_{2}^{2}\right]
\leq
\left(\alpha\widehat{\sigma}_{\max}+\varepsilon\right)\|\Delta WP\|_{F}^{2}.
\end{equation}
\end{corollary}

Corollary~\ref{cor:population-extension} indicates that a projector constructed from representative expert-task calibration data controls expected perturbations on unseen activations from the same distribution. MergeHEIR thus limits adaptation-induced changes to expert-task representations beyond the calibration set.

Together, these results provide a theoretical account of how null-space projection in MergeHEIR controls perturbations to expert-task representations during hallucination-oriented adaptation. The projection minimizes update distortion within the selected subspace, which is maximum-dimensional under the prescribed sensitivity constraint. The threshold analysis characterizes the balance between adaptation flexibility and representation stability, while the population extension carries perturbation control beyond calibration data under the stated spectral approximation condition. Complete proofs and fixed-prefix output analysis are provided in Appendix~\ref{app:theory}.

%% file: secs/5_experiments.tex
\section{Experiments}
\label{sec:experiments}

\subsection{Experimental setup}
\label{sec:experimental-setup}

\paragraph{Models, merging methods, and data.}
We conduct experiments on three MLLM configurations: Qwen2.5-VL-3B, Qwen2.5-VL-7B \citep{bai2025qwen2}, and InternVL3.5-8B \citep{wang2025internvl3}. For each configuration, we merge four task-specialized experts covering chart understanding, geometry reasoning, OCR, and visual question answering using eight representative methods: Task Arithmetic, TIES, DELLA, Breadcrumbs, Model Stock, Iso-C, WUDI, and OptMerge. Post-merge adaptation uses 10,000 multimodal preference pairs sampled from RLAIF-V \citep{yu2025rlaif}, while layer-wise null-space projectors are constructed from 500 calibration examples per expert task.

\paragraph{Evaluation.}
We evaluate both the retention of inherited expertise and the mitigation of hallucination. Expert-task performance is measured on ChartQA \citep{masry2022chartqa} for chart understanding; MATH-Vision \citep{wang2024measuring} and Geometry3K \citep{lu2021inter} for geometry reasoning; TextVQA \citep{singh2019towards} and DocVQA \citep{mathew2021docvqa} for OCR; and GQA \citep{hudson2019gqa} and OKVQA \citep{marino2019ok} for visual question answering. We report the unweighted mean over these seven benchmarks as \textit{Expert Avg.} Hallucination behavior is evaluated using AMBER \citep{wang2023amber} for multidimensional generative hallucination, Object HalBench \citep{rohrbach2018object} for object hallucination in open-ended image descriptions, POPE \citep{li2023evaluating} for object-existence judgments, and GPT-4-assisted evaluation \citep{zhao2023beyond} for fine-grained hallucinations in generated image descriptions. The latter reports the Sentence-level Hallucination Ratio (SHR), where lower values indicate fewer hallucinated sentences. To reduce sampling-induced variance and ensure consistent checkpoint-level comparisons across all benchmarks, we generate every model response with greedy decoding throughout evaluation.

\paragraph{Implementation details.}
We perform LoRA-based post-merge adaptation for two epochs with a global batch size of 32 and a DPO scaling coefficient of $\beta=0.1$. LoRA is applied only to non-visual modules, while all visual-side parameters remain frozen. For Qwen2.5-VL-3B, we use $(r,\alpha_{\mathrm{LoRA}})=(32,64)$ and a learning rate of $2\times10^{-4}$; for Qwen2.5-VL-7B and InternVL3.5-8B, we use $(r,\alpha_{\mathrm{LoRA}})=(16,32)$ and a learning rate of $1\times10^{-4}$. We set $\lambda_{\mathrm{img}}=0.2$, $\lambda_{\mathrm{anc}}=1.0$, and $\delta=0$, with a random image-masking ratio of $0.3$. The accumulated parameter displacement is projected every five optimizer steps using a relative spectral threshold of $\alpha=5\times10^{-4}$. Additional experimental details are provided in Appendix~\ref{app:experimental_details}.

\begin{figure*}[!t]
\centering
\includegraphics[width=0.98\textwidth]{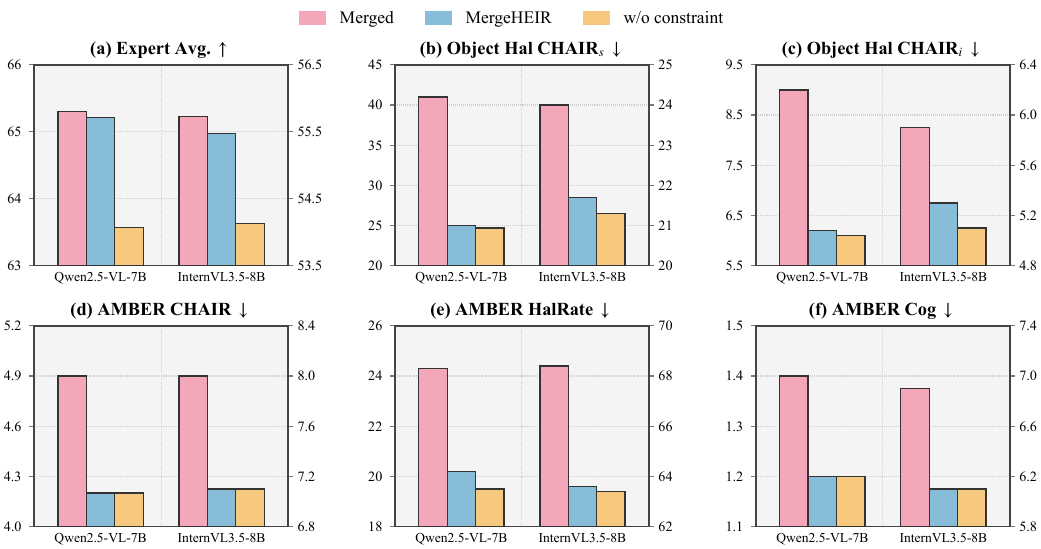}
\caption{Constraint ablation on Qwen2.5-VL-7B and InternVL3.5-8B. MergeHEIR better preserves inherited expertise while achieving hallucination mitigation comparable to unconstrained adaptation.}
\label{fig:constraint-ablation}
\end{figure*}
\vspace{-3pt}

\subsection{Main results}
\label{sec:main-results}

Tables~\ref{tab:qwen-expert} and~\ref{tab:qwen-hallucination} report the complete Qwen2.5-VL results across all eight merging methods. Each method is evaluated before and after post-merge adaptation, enabling direct paired comparison. Appendix~\ref{app:additional-results} reports complete InternVL3.5-8B evaluations, POPE results across all three splits, and compatibility experiments with RLAIF-V \citep{yu2025rlaif} and OPA-DPO \citep{yang2025mitigating}. Together, these results examine whether the retention benefit of null-space projection generalizes across model families, evaluation settings, and alternative preference objectives.

\paragraph{Consistent hallucination mitigation.}
MergeHEIR reduces AMBER HalRate in all 24 matched comparisons across three model configurations and eight merging methods. The improvement is most pronounced on Qwen2.5-VL-3B, for which HalRate decreases from $41.0$--$59.6$ before adaptation to $20.4$--$25.2$ afterward. Consistent reductions are also observed on Qwen2.5-VL-7B and InternVL3.5-8B. Object HalBench further exhibits substantial reductions in both CHAIR$_s$ and CHAIR$_i$ across most configurations. POPE accuracy and F1 improve across all evaluated splits and all 24 comparisons. Finally, the GPT-4-assisted evaluation provides complementary judge-based evidence: SHR decreases in all 24 matched comparisons, with mean reductions of $2.80$, $2.76$, and $3.26$ points on Qwen2.5-VL-3B, Qwen2.5-VL-7B, and InternVL3.5-8B, respectively. Together, these results demonstrate that MergeHEIR consistently mitigates hallucination across complementary generative, discriminative, and judge-based evaluations.

\paragraph{Preservation of inherited expertise.}
The broad hallucination improvements are achieved with limited changes to the expert capabilities inherited by the merged checkpoints. On Qwen2.5-VL-3B, MergeHEIR achieves substantial hallucination mitigation while leaving Expert Avg essentially unchanged. Qwen2.5-VL-7B and InternVL3.5-8B exhibit the same favorable pattern, with broad hallucination improvements accompanied by mean Expert Avg decreases of only $0.28$ and $0.31$ points, respectively. These results demonstrate that MergeHEIR delivers broad hallucination mitigation across model families without substantially compromising the capabilities consolidated during merging. The matched ablation in Figure~\ref{fig:constraint-ablation} tests whether null-space projection improves expert retention without materially weakening hallucination mitigation on Qwen2.5-VL-7B and InternVL3.5-8B.

\subsection{Ablation studies}
\label{sec:ablation}

\begin{figure*}[!t]
\centering
\includegraphics[width=0.98\textwidth]{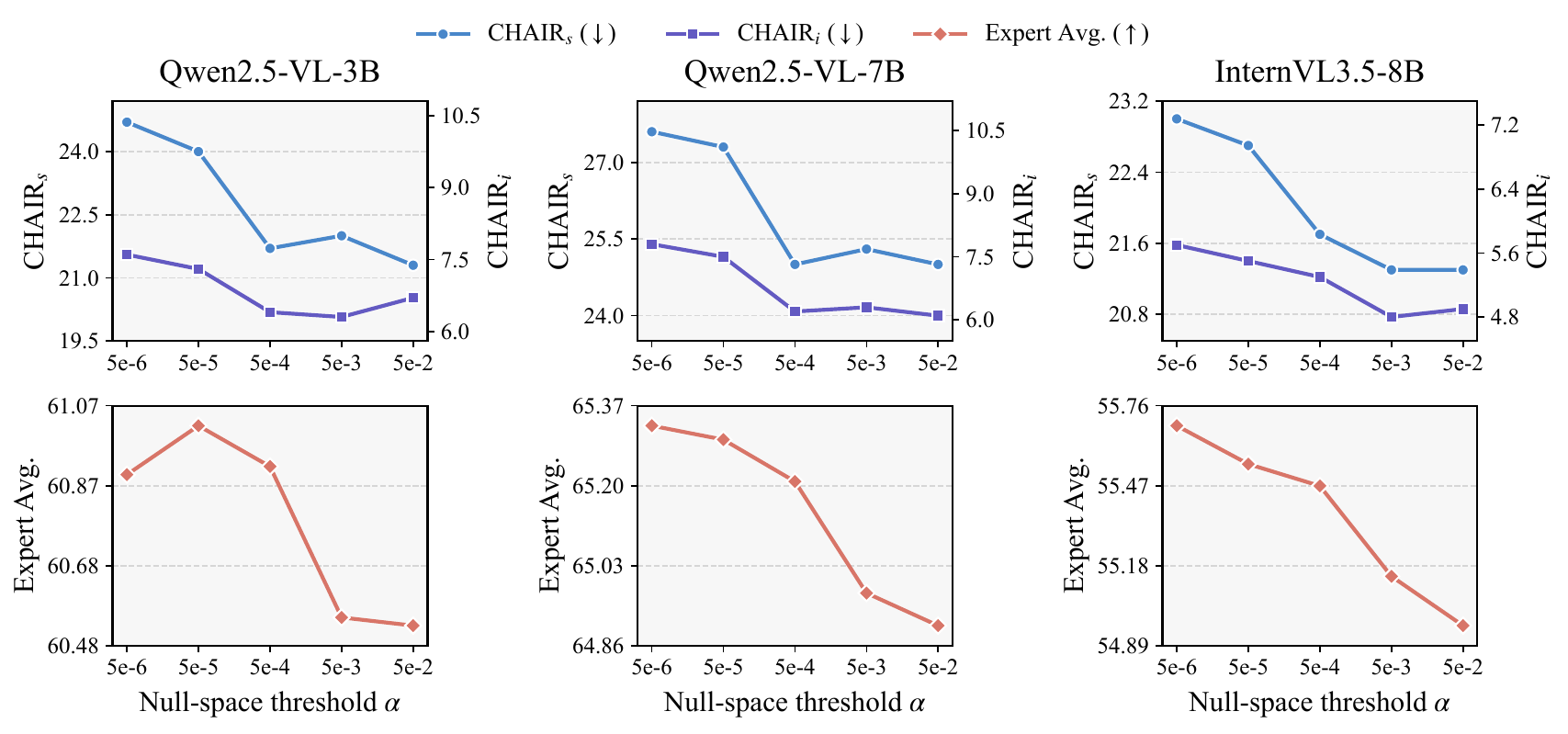}
\caption{Sensitivity to the null-space threshold $\alpha$ across three model configurations. The upper panels report Object HalBench CHAIR$_s$ and CHAIR$_i$, while the lower panels report Expert Avg.}
\label{fig:threshold-sensitivity}
\end{figure*}

\paragraph{Effect of null-space projection.}
We conduct matched ablations on two MLLM configurations. Within each configuration, both variants use the same training setup, differing only in null-space projection. Across these settings, removing null-space projection decreases Expert Avg.\ by $1.64$ and $1.34$ points relative to MergeHEIR on Qwen2.5-VL-7B and InternVL3.5-8B, respectively, while yielding comparable hallucination performance. This pronounced asymmetry shows that unconstrained adaptation incurs substantially greater damage to inherited expertise for only marginal behavioral gains. By restricting the accumulated post-merge displacement to the expert-task null spaces, MergeHEIR achieves a markedly better hallucination--retention trade-off, preserving expert performance without forfeiting the primary hallucination improvements.

\paragraph{Sensitivity to the null-space threshold.}
Figure~\ref{fig:threshold-sensitivity} illustrates the effect of the relative spectral threshold $\alpha$ on hallucination mitigation and expert retention. Across the three model configurations, larger values of $\alpha$ generally yield lower CHAIR$_s$ and CHAIR$_i$, whereas smaller values preserve higher Expert Avg. This trend is consistent with Proposition~\ref{prop:threshold-geometry}, showing that $\alpha$ controls the trade-off between hallucination mitigation and inherited-expertise preservation. We therefore use $\alpha=5\times10^{-4}$ as the shared default because it achieves substantial hallucination mitigation while limiting degradation of inherited expert capabilities across model configurations. 

%% file: secs/6_conclusion.tex
\section{Conclusion}
This work investigates hallucination in multimodal model merging. We begin by identifying two coupled challenges: a systematic merging tax and expertise interference during post-merge correction, based on which we propose MergeHEIR, combining multimodal preference optimization with activation-derived null-space projection. We further analyze the mechanism underlying the merging tax and establish guarantees on adaptation geometry and representation perturbation control. Extensive experiments across three MLLM configurations, eight merging methods, and diverse benchmarks validate its effectiveness and generality. We highlight the importance of coordinating hallucination mitigation with inherited-expertise preservation for reliable model merging.

%% file: secs/2_related_work.tex
\section{Related Work}

\textbf{Model merging.}
Model merging consolidates independently fine-tuned models into a shared checkpoint without incurring the inference cost of an ensemble. Static merging methods include interpolation- and arithmetic-based approaches, such as model soups \citep{wortsman2022model} and Task Arithmetic \citep{ilharco2023editing}; sparsification-based approaches, including TIES \citep{yadav2023ties}, DARE \citep{yu2023language}, DELLA \citep{deep2024della}, and Model Breadcrumbs \citep{davari2024model}; and geometry-based or subspace-based approaches, such as Model Stock \citep{jang2024model} and Iso-C \citep{marczak2025no}. Optimization-based methods instead learn merging coefficients or optimize task-vector combinations, including AdaMerging \citep{yang2024adamerging}, WUDI \citep{cheng2025whoever}, and OptMerge \citep{wei2025optmerge}. Recent methods further extend model merging to heterogeneous MLLM capabilities and modalities \citep{du2025adamms,wei2025optmerge}. The merging methods evaluated in our study span these methodological families. MergeHEIR takes the resulting merged checkpoints as its starting point and studies how to mitigate the hallucination degradation introduced by capability consolidation while retaining inherited expertise.

\textbf{Hallucination mitigation in MLLMs.}
Multimodal hallucination refers to generated content that is inconsistent with the visual input. Existing mitigation methods can be broadly divided into training-free and training-based approaches. Training-free methods intervene during inference through contrastive or adaptive decoding, prompting, and post-hoc correction \citep{leng2024mitigating,chen2024halc,yin2024woodpecker,fang2026grounding}. Training-based approaches instead improve visual grounding using hallucination-aware supervision or preference data. RLAIF-V \citep{yu2025rlaif} constructs multimodal preferences from AI feedback, while mDPO \citep{wang2024mdpo} strengthens image dependence through image-conditioned preference optimization. HDPO \citep{fu2025mitigating} and EMPO \citep{wu2025mitigating} further target heterogeneous hallucination causes and fine-grained modality alignment. Recent work also studies hallucination in multimodal reasoning models \citep{dong2025mirage,fang2026seeing,kong2026reasoning} and multi-image settings \citep{yang2026looking}.

\textbf{Null-space constrained adaptation.}
Null-space constraints have been used to mitigate interference in continual learning and model editing \citep{wang2021training,fang2025alphaedit}. Recent work extends this idea to LLM alignment and merging: NSPO projects safety gradients using general-task representations \citep{niu2026mitigating}, while RAIN-Merging projects task vectors using reasoning-related forward features \citep{huang2026rain}. MergeHEIR instead targets hallucination-oriented adaptation after heterogeneous multimodal expert merging, constructing a shared protected subspace from pooled expert-task activations and periodically projecting the accumulated post-merge displacement. We further provide geometric and perturbation guarantees for this projection.

%% file: secs/A_merging_tax_analysis.tex
\section{Mechanistic Interpretation of the Merging Tax}
\label{app:merging-tax-mechanism}

We examine the mechanism underlying the \emph{merging tax}, where model merging increases multimodal hallucination. All constituent experts share the same frozen visual encoder and differ only in their language-side adaptations. For a fixed image, the encoded visual representation is therefore unchanged before and after merging, locating the source of the degradation downstream in how the language-side network uses visual evidence during generation.

Multimodal hallucination occurs when a response makes claims about the visual input that are unsupported by, or inconsistent with, the available visual evidence. During expert specialization, language-side parameters are optimized jointly for each task, forming co-adapted computations that govern how visual cues are propagated, integrated with textual context, and translated into token predictions. Model merging recombines these independently learned adaptations into a new parameter configuration. This recomposition can alter the cross-layer interactions that determine how strongly visual evidence influences generation, even when the visual representation itself remains unchanged. We refer to this disruption in the downstream use of visual evidence as \emph{visual-evidence utilization interference}.

The constituent experts inherit the same pretrained semantic and linguistic structure, while task specialization adapts how this shared knowledge is conditioned on visual input. When merging weakens these image-conditioned adaptations, generation becomes less tightly coupled to the current image and more susceptible to semantic regularities that are not visually supported. Unsupported objects, attributes, quantities, relations, or text can therefore enter the response more readily. The systematic increase in hallucination observed in Figure~\ref{fig:preliminary-merge-hallucination} is consistent with this interpretation, suggesting a shared failure mode in which visual evidence exerts weaker control over generation despite differences among merging methods.

%% file: secs/A_theoretical_analysis.tex
\section{Additional Theoretical Analysis}
\label{app:theory}

This section provides complete proofs for the results in Section~\ref{sec:geometric-analysis}, clarifying how null-space projection in MergeHEIR balances adaptation flexibility and expert-task representation stability. We begin with the projection update and spectral subspace selection, then examine the role of the threshold. The spectral perturbation bound is further extended in two directions: to the corresponding activation distribution under spectral approximation, and to fixed-prefix outputs under local smoothness.

\paragraph{Notation and normalization.}
Consider one adapted linear module with merged weight $W_m\in\mathbb{R}^{d_{\mathrm{out}}\times d_{\mathrm{in}}}$ and accumulated displacement $\Delta W$. Let $H=[H_t]_{t\in\mathcal{T}}\in\mathbb{R}^{d_{\mathrm{in}}\times N}$ concatenate the expert-task calibration activations, and write
\begin{equation}
G=HH^{\top}
=
U\operatorname{diag}(\sigma_1,\ldots,\sigma_{d_{\mathrm{in}}})U^{\top},
\qquad
0\leq\sigma_1\leq\cdots\leq\sigma_{d_{\mathrm{in}}}=\sigma_{\max}.
\end{equation}
Define
\begin{equation}
\sigma_{\mathrm{cut}}=\alpha\sigma_{\max},
\qquad
\mathcal{I}_{\alpha}=\{i:\sigma_i<\sigma_{\mathrm{cut}}\},
\qquad
U_{\mathrm{null}}=[u_i\mid i\in\mathcal{I}_{\alpha}],
\qquad
P=U_{\mathrm{null}}U_{\mathrm{null}}^{\top}.
\end{equation}
Thus $P=P^{\top}=P^2$ and $PG=GP$. For $\widehat{G}=\frac{1}{N}G$, its eigenvalues satisfy $\widehat{\sigma}_i=\frac{\sigma_i}{N}$ and $\widehat{\sigma}_{\max}=\frac{\sigma_{\max}}{N}$. Hence
\begin{equation}
\sigma_i<\sigma_{\mathrm{cut}}
\quad\Longleftrightarrow\quad
\widehat{\sigma}_i<\alpha\widehat{\sigma}_{\max},
\end{equation}
so normalization leaves $P$ unchanged, allowing the same projector to be analyzed using the empirical second moment in the population extension.

\subsection{Proof of Proposition~\ref{prop:min-distortion}}
\label{app:proof-min-distortion}
We first justify the projection of the accumulated displacement onto a fixed admissible subspace $\mathcal{S}_P=\{Z:Z=ZP\}$. The orthogonal decomposition induced by $P$ establishes that $\Delta WP$ minimizes update distortion in Frobenius norm.
\begin{proof}
For any $Z\in\mathcal{S}_{P}$,
\begin{equation}
Z-\Delta W
=
(Z-\Delta WP)-\Delta W(I-P).
\end{equation}
Since $Z-\Delta WP=(Z-\Delta WP)P$, the two terms are Frobenius-orthogonal:
\begin{align}
\left\langle Z-\Delta WP,\Delta W(I-P)\right\rangle_F
&=
\operatorname{tr}\!\left(P(Z-\Delta WP)^{\top}\Delta W(I-P)\right)\\
&=
\operatorname{tr}\!\left((I-P)P(Z-\Delta WP)^{\top}\Delta W\right)
=0.
\end{align}
Therefore,
\begin{equation}
\|Z-\Delta W\|_F^2
=
\|Z-\Delta WP\|_F^2
+
\|\Delta W(I-P)\|_F^2.
\end{equation}
The second term is independent of $Z$, whereas the first is uniquely minimized at $Z=\Delta WP$. Hence, $\Delta WP$ is the unique orthogonal projection of $\Delta W$ onto $\mathcal{S}_{P}$ under the Frobenius inner product. Setting $Z=0$ in the preceding decomposition yields
\begin{equation}
\|\Delta W\|_F^2
=
\|\Delta WP\|_F^2
+
\|\Delta W(I-P)\|_F^2,
\end{equation}
and therefore $\|\Delta WP\|_F\leq\|\Delta W\|_F$.
\end{proof}

\subsection{Proof of Theorem~\ref{thm:spectral-control}}
\label{app:proof-spectral}
The preceding result justifies the projection within a fixed subspace. To relate this geometric property to expert-task stability, we analyze the spectral construction, establishing maximum dimensionality under the sensitivity constraint and then deriving pooled and task-specific perturbation bounds.
\begin{proof}
\noindent\textit{Step 1: Maximum-dimensional admissible subspace.}
Let $k=|\mathcal{I}_{\alpha}|=\operatorname{rank}(P)$. Every nonzero vector $v\in\operatorname{range}(P)$ can be written as
\begin{equation}
v=\sum_{i\in\mathcal{I}_{\alpha}}c_i u_i.
\end{equation}
Using the orthonormality of the eigenvectors,
\begin{equation}
v^{\top}Gv
=
\sum_{i\in\mathcal{I}_{\alpha}}\sigma_i c_i^2
<
\sigma_{\mathrm{cut}}
\sum_{i\in\mathcal{I}_{\alpha}}c_i^2
=
\sigma_{\mathrm{cut}}\|v\|_2^2.
\end{equation}
Thus, $\operatorname{range}(P)$ satisfies the prescribed activation-sensitivity condition.

Now consider any subspace $\mathcal{V}\subseteq\mathbb{R}^{d_{\mathrm{in}}}$ with $\dim(\mathcal{V})>k$, and define
\begin{equation}
\mathcal{U}_{\mathrm{high}}
=
\operatorname{span}\{u_i:\sigma_i\geq\sigma_{\mathrm{cut}}\}.
\end{equation}
Since $\dim(\mathcal{U}_{\mathrm{high}})=d_{\mathrm{in}}-k$, the dimension inequality
\begin{equation}
\dim(\mathcal{V})+\dim(\mathcal{U}_{\mathrm{high}})
>
d_{\mathrm{in}}
\end{equation}
implies that $\mathcal{V}\cap\mathcal{U}_{\mathrm{high}}$ contains a nonzero vector. For any such vector $v$,
\begin{equation}
v^{\top}Gv
\geq
\sigma_{\mathrm{cut}}\|v\|_2^2.
\end{equation}
Hence, no subspace of dimension greater than $k$ can satisfy the strict sensitivity condition for every nonzero vector. Since $\operatorname{range}(P)$ has dimension $k$ and satisfies this condition, it is maximum-dimensional.

\noindent\textit{Step 2: Pooled and task-specific perturbation bounds.}
We now translate the spectral sensitivity constraint into a bound on the representation perturbation induced by the projected displacement. For the pooled activation matrix,
\begin{align}
\|\Delta WPH\|_F^2
&=
\operatorname{tr}\!\left(
\Delta WPHH^{\top}P\Delta W^{\top}
\right)\\
&=
\operatorname{tr}\!\left(
\Delta WPGP\Delta W^{\top}
\right).
\end{align}
By the spectral definition of $P$,
\begin{equation}
PGP
=
\sum_{i\in\mathcal{I}_{\alpha}}
\sigma_i u_i u_i^{\top}
\preceq
\sigma_{\mathrm{cut}}
\sum_{i\in\mathcal{I}_{\alpha}}
 u_i u_i^{\top}
=
\sigma_{\mathrm{cut}}P.
\end{equation}
Therefore,
\begin{align}
\|\Delta WPH\|_F^2
&\leq
\sigma_{\mathrm{cut}}
\operatorname{tr}\!\left(
\Delta WP\Delta W^{\top}
\right)\\
&=
\sigma_{\mathrm{cut}}
\|\Delta WP\|_F^2.
\end{align}

To obtain task-specific bounds from the pooled construction, let $G_t=H_tH_t^{\top}$. Since $H$ concatenates all task-specific activation matrices,
\begin{equation}
G=\sum_{t\in\mathcal{T}}G_t,
\qquad
0\preceq G_t\preceq G.
\end{equation}
Applying the congruence transformation induced by $P$ gives
\begin{equation}
PG_tP
\preceq
PGP
\preceq
\sigma_{\mathrm{cut}}P.
\end{equation}
It follows that
\begin{align}
\|\Delta WPH_t\|_F^2
&=
\operatorname{tr}\!\left(
\Delta WPG_tP\Delta W^{\top}
\right)\\
&\leq
\sigma_{\mathrm{cut}}
\operatorname{tr}\!\left(
\Delta WP\Delta W^{\top}
\right)\\
&=
\sigma_{\mathrm{cut}}
\|\Delta WP\|_F^2.
\end{align}
Substituting $\sigma_{\mathrm{cut}}=\alpha\sigma_{\max}$ yields the stated pooled and task-specific bounds.
\end{proof}

\subsection{Proof of Proposition~\ref{prop:threshold-geometry}}
\label{app:threshold-monotonicity}
The spectral threshold determines both the admissible displacement space and its perturbation bound. We characterize this balance by comparing the retained update and calibration perturbation for a fixed displacement across nested spectral subspaces.
\begin{proof}
Let
\begin{equation}
\mathcal{I}_{\alpha}
=
\{i:\sigma_i<\alpha\sigma_{\max}\}.
\end{equation}
Since $\alpha_1\leq\alpha_2$,
\begin{equation}
\mathcal{I}_{\alpha_1}
\subseteq
\mathcal{I}_{\alpha_2}.
\end{equation}
Therefore,
\begin{equation}
Q
=
\sum_{i\in\mathcal{I}_{\alpha_2}\setminus\mathcal{I}_{\alpha_1}}
u_i u_i^{\top}
\end{equation}
is an orthogonal projector satisfying
\begin{equation}
P_2=P_1+Q,
\qquad
P_1Q=QP_1=0.
\end{equation}
The range inclusion, projector identities, Loewner ordering, and rank inequality follow immediately.

Using the orthogonality of $P_1$ and $Q$,
\begin{equation}
\|\Delta WP_2\|_F^2
=
\|\Delta W(P_1+Q)\|_F^2
=
\|\Delta WP_1\|_F^2
+
\|\Delta WQ\|_F^2.
\end{equation}
Likewise, since
\begin{equation}
I-P_1=(I-P_2)+Q
\end{equation}
and $(I-P_2)Q=0$, we obtain
\begin{equation}
\|\Delta W(I-P_1)\|_F^2
=
\|\Delta W(I-P_2)\|_F^2
+
\|\Delta WQ\|_F^2.
\end{equation}

Finally, $P_1$ and $Q$ project onto disjoint eigenspaces of $G$, so
\begin{equation}
P_1GQ=QGP_1=0.
\end{equation}
Hence,
\begin{align}
\|\Delta WP_2H\|_F^2
&=
\operatorname{tr}\!\left(
\Delta W(P_1+Q)G(P_1+Q)\Delta W^{\top}
\right)\\
&=
\operatorname{tr}\!\left(
\Delta WP_1GP_1\Delta W^{\top}
\right)
+
\operatorname{tr}\!\left(
\Delta WQGQ\Delta W^{\top}
\right)\\
&=
\|\Delta WP_1H\|_F^2
+
\|\Delta WQH\|_F^2.
\end{align}
\end{proof}

By relating the retained update, projection-induced distortion, and calibration perturbation, these identities characterize how $\alpha$ balances adaptation flexibility and representation stability.

Because the retained index set changes only when $\alpha\sigma_{\max}$ crosses an eigenvalue of $G$, the projector $P_{\alpha}$ is piecewise constant between the normalized spectral breakpoints $\{\sigma_i/\sigma_{\max}\}_{i=1}^{d_{\mathrm{in}}}$. This spectral structure also explains why threshold sensitivity can exhibit plateau-like regions.

\subsection{Proof of Corollary~\ref{cor:population-extension}}
\label{app:proof-population-extension}
The calibration bound controls perturbations through the empirical second moment $\widehat{G}$. To extend this control to the corresponding activation distribution, we compare $\widehat{G}$ with $\Gamma$ and account for their spectral discrepancy.
\begin{proof}
The condition
\begin{equation}
\|\widehat{G}-\Gamma\|_2\leq\varepsilon
\end{equation}
implies the Loewner-order bounds
\begin{equation}
-\varepsilon I
\preceq
\Gamma-\widehat{G}
\preceq
\varepsilon I.
\end{equation}
Applying the congruence transformation induced by $P$ and using $P^2=P$ gives
\begin{equation}
P\Gamma P
\preceq
P\widehat{G}P+\varepsilon P.
\end{equation}
Since $P$ retains only eigenvectors of $\widehat{G}$ whose eigenvalues are below $\alpha\widehat{\sigma}_{\max}$,
\begin{equation}
P\widehat{G}P
\preceq
\alpha\widehat{\sigma}_{\max}P.
\end{equation}
Consequently,
\begin{equation}
P\Gamma P
\preceq
\left(
\alpha\widehat{\sigma}_{\max}
+
\varepsilon
\right)P.
\end{equation}
For an activation vector $h$ drawn from the corresponding distribution, the definition $\Gamma=\mathbb{E}[hh^{\top}]$ gives
\begin{align}
\mathbb{E}\!\left[
\|\Delta WPh\|_2^2
\right]
&=
\operatorname{tr}\!\left(
\Delta WP\Gamma P\Delta W^{\top}
\right)\\
&\leq
\left(
\alpha\widehat{\sigma}_{\max}
+
\varepsilon
\right)
\operatorname{tr}\!\left(
\Delta WP\Delta W^{\top}
\right)\\
&=
\left(
\alpha\widehat{\sigma}_{\max}
+
\varepsilon
\right)
\|\Delta WP\|_F^2.
\end{align}
\end{proof}

Corollary~\ref{cor:population-extension} extends the empirical perturbation bound to the corresponding activation distribution, with the spectral discrepancy entering the bound additively through $\varepsilon$. Thus, the thresholded null-space construction provides representation-level perturbation control beyond the collected calibration activations under the stated spectral condition.

\subsection{From Representation Perturbations to Fixed-Prefix Outputs}
\label{app:fixed-prefix-logits}

To relate representation stability to model behavior, we combine the module-level bound in Theorem~\ref{thm:spectral-control} with local smoothness assumptions. This yields perturbation bounds for logits and next-token distributions under fixed calibration prefixes. Let $\theta_m$ denote the merged checkpoint, and let $\boldsymbol{\Delta}_P$ denote the model-wide projected displacement whose $\ell$-th adapted block is $\Delta W_{\ell}P_{\ell}$. Define
\begin{equation}
\|\boldsymbol{\Delta}_P\|_{\mathrm{blk}}^2
=
\sum_{\ell=1}^{L}
\|\Delta W_{\ell}P_{\ell}\|_F^2,
\label{eq:block-displacement-norm}
\end{equation}
where $L$ is the number of adapted modules. Let $H_{\ell}$ collect the inputs to module $\ell$ under the merged checkpoint $\theta_m$ at the evaluated positions, and let $Z_{\theta}$ denote the corresponding stacked logits.

\paragraph{Local output smoothness.}
To quantify how module-level perturbations propagate to the logits, we use a first-order Taylor expansion with a quadratic remainder \citep{nesterov2013introductory} and bounded local sensitivities \citep{fazlyab2019efficient}. For a fixed collection of calibration prefixes, assume that the logit map admits the following local expansion around $\theta_m$:
\begin{equation}
Z_{\theta_m+\boldsymbol{\Delta}_P}
-
Z_{\theta_m}
=
\sum_{\ell=1}^{L}
\mathcal{J}_{\ell}
\!\left[
\Delta W_{\ell}P_{\ell}H_{\ell}
\right]
+
\mathcal{R}(\boldsymbol{\Delta}_P),
\label{eq:local-logit-expansion}
\end{equation}
where $\mathcal{J}_{\ell}$ is a bounded linear map describing the local response of the stacked logits to the output perturbation at module $\ell$, with
\begin{equation}
\|\mathcal{J}_{\ell}[X]\|_F
\leq
\kappa_{\ell}\|X\|_F.
\label{eq:local-operator-bound}
\end{equation}
Assume further that the nonlinear remainder satisfies
\begin{equation}
\|\mathcal{R}(\boldsymbol{\Delta}_P)\|_F
\leq
C\|\boldsymbol{\Delta}_P\|_{\mathrm{blk}}^2.
\label{eq:local-remainder-bound}
\end{equation}
For compactness, define
\begin{equation}
\mathcal{B}_{\alpha}(\boldsymbol{\Delta}_P)
=
\sum_{\ell=1}^{L}
\kappa_{\ell}
\sqrt{\sigma_{\mathrm{cut}}^{(\ell)}}
\|\Delta W_{\ell}P_{\ell}\|_F
+
C\|\boldsymbol{\Delta}_P\|_{\mathrm{blk}}^2,
\label{eq:fixed-prefix-bound}
\end{equation}
where $\sigma_{\mathrm{cut}}^{(\ell)}=\alpha\sigma_{\max}^{(\ell)}$.

\paragraph{Fixed-prefix output control.}
Under the local smoothness condition above, the logit perturbation satisfies
\begin{equation}
\left\|
Z_{\theta_m+\boldsymbol{\Delta}_P}
-
Z_{\theta_m}
\right\|_F
\leq
\mathcal{B}_{\alpha}(\boldsymbol{\Delta}_P).
\label{eq:fixed-prefix-logit-bound}
\end{equation}
Moreover, if $p_j$ and $q_j$ denote the next-token distributions induced by the merged and post-projection models, respectively, at fixed prefix $j$, then
\begin{equation}
\sum_j
D_{\mathrm{KL}}(p_j\|q_j)
\leq
\frac{1}{4}
\mathcal{B}_{\alpha}(\boldsymbol{\Delta}_P)^2.
\label{eq:fixed-prefix-kl-bound}
\end{equation}

\paragraph{Derivation.}
Applying the triangle inequality, the local operator bounds, and Theorem~\ref{thm:spectral-control} module-wise gives
\begin{align}
\left\|
Z_{\theta_m+\boldsymbol{\Delta}_P}
-
Z_{\theta_m}
\right\|_F
&\leq
\sum_{\ell=1}^{L}
\kappa_{\ell}
\|\Delta W_{\ell}P_{\ell}H_{\ell}\|_F
+
\|\mathcal{R}(\boldsymbol{\Delta}_P)\|_F\\
&\leq
\sum_{\ell=1}^{L}
\kappa_{\ell}
\sqrt{\sigma_{\mathrm{cut}}^{(\ell)}}
\|\Delta W_{\ell}P_{\ell}\|_F
+
C\|\boldsymbol{\Delta}_P\|_{\mathrm{blk}}^2\\
&=
\mathcal{B}_{\alpha}(\boldsymbol{\Delta}_P),
\end{align}
which establishes Equation~\eqref{eq:fixed-prefix-logit-bound}.

To translate logit perturbations into changes in next-token distributions, we use the log-partition function
\begin{equation}
f(z)=\log\sum_i\exp(z_i).
\end{equation}
For categorical exponential families, the KL divergence can be expressed as the Bregman divergence induced by the log-partition function \citep{banerjee2005clustering,wainwright2008graphical}.
For $\rho=\operatorname{softmax}(z)$, its Hessian is
\begin{equation}
\nabla^2f(z)
=
\operatorname{diag}(\rho)-\rho\rho^{\top}.
\end{equation}
For any vector $x$,
\begin{equation}
x^{\top}\nabla^2f(z)x
=
\operatorname{Var}_{i\sim\rho}(x_i)
\leq
\frac{(\max_i x_i-\min_i x_i)^2}{4}
\leq
\frac{1}{2}\|x\|_2^2.
\end{equation}
Hence $\|\nabla^2f(z)\|_2\leq1/2$, so $f$ is $1/2$-smooth. For $p=\operatorname{softmax}(z)$ and $q=\operatorname{softmax}(z+u)$,
\begin{align}
D_{\mathrm{KL}}(p\|q)
&=
f(z+u)-f(z)-\nabla f(z)^{\top}u\\
&\leq
\frac{1}{4}\|u\|_2^2.
\end{align}
Applying this inequality at each fixed prefix and using the fact that $Z_{\theta}$ stacks the corresponding logit vectors gives
\begin{align}
\sum_jD_{\mathrm{KL}}(p_j\|q_j)
&\leq
\frac{1}{4}
\left\|
Z_{\theta_m+\boldsymbol{\Delta}_P}
-
Z_{\theta_m}
\right\|_F^2\\
&\leq
\frac{1}{4}
\mathcal{B}_{\alpha}(\boldsymbol{\Delta}_P)^2,
\end{align}
which establishes Equation~\eqref{eq:fixed-prefix-kl-bound}.

\paragraph{Leading-order implication.}
The first term in $\mathcal{B}_{\alpha}(\boldsymbol{\Delta}_P)$ is of order $\mathcal{O}(\|\boldsymbol{\Delta}_P\|_{\mathrm{blk}})$. Therefore, as $\|\boldsymbol{\Delta}_P\|_{\mathrm{blk}}\to0$, expanding Equation~\eqref{eq:fixed-prefix-kl-bound} yields
\begin{equation}
\sum_jD_{\mathrm{KL}}(p_j\|q_j)
\leq
\frac{1}{4}
\left(
\sum_{\ell=1}^{L}
\kappa_{\ell}
\sqrt{\sigma_{\mathrm{cut}}^{(\ell)}}
\|\Delta W_{\ell}P_{\ell}\|_F
\right)^2
+
\mathcal{O}\!\left(
\|\boldsymbol{\Delta}_P\|_{\mathrm{blk}}^3
\right).
\label{eq:fixed-prefix-leading-order}
\end{equation}
Thus, the fixed-prefix distribution shift admits a quadratic leading-order upper bound in the retained displacement, jointly governed by the downstream sensitivities $\{\kappa_{\ell}\}_{\ell=1}^{L}$ and the layer-wise spectral cutoffs. This extends the representation-level perturbation analysis to the model's local output behavior under shared calibration contexts.

\paragraph{Theoretical implications.}
\label{app:theory-scope}
Taken together, these results connect the geometry of MergeHEIR's projected updates to representation stability and local output behavior.  The projection minimizes update distortion within a maximum-dimensional subspace under the prescribed sensitivity constraint, while $\alpha$ balances adaptation flexibility and calibration perturbation.  The population extension accounts for spectral approximation error, and the fixed-prefix analysis propagates module-level perturbation bounds to logits and next-token distributions under local smoothness.  These results explain how projection limits interference with expert-task representations during hallucination-oriented adaptation, complementing the empirical evaluation of expert retention.

%% file: secs/A_experimental_details.tex
\section{Experimental Details}
\label{app:experimental_details}

\subsection{Training Task-Specialized Experts}
\label{app:expert-training}

\paragraph{Training data.}
For each backbone, we independently train four task-specialized experts for chart understanding, geometry reasoning, OCR, and visual question answering. Table~\ref{tab:expert-training-data} summarizes the datasets and training-set sizes used for each expert.

\begin{table*}[h]
\vspace{-1.5em}
\caption{Summary of the training datasets used to construct the four task-specialized experts.}
\label{tab:expert-training-data}
\centering
\small
\renewcommand{\arraystretch}{1.15}
\setlength{\tabcolsep}{6pt}
\arrayrulecolor{gray!60}
\begin{tabular}{
    >{\columncolor{gray!12}\bfseries\raggedright\arraybackslash}l|
    >{\centering\arraybackslash}c|
    >{\raggedright\arraybackslash}p{0.72\textwidth}
}
\hline
\rowcolor{gray!25}
\textbf{Task Category} & \textbf{Size} & \textbf{Training Datasets} \\
\hline
VQA & 589K & GQA \citep{hudson2019gqa}, CogVLM-Multiround \citep{wang2024cogvlm}, CogVLM-Singleround \citep{wang2024cogvlm} \\
\hline
Geometry & 177K & GeoQA+ \citep{cao2022augmented}, G-LLaVA \citep{gao2025g} \\
\hline
Chart & 219K & ChartQA \citep{masry2022chartqa}, DVQA \citep{kafle2018dvqa} \\
\hline
OCR & 239K & OCRVQA \citep{mishra2019ocr}, SynthDoG \citep{kim2022ocr}, DocVQA \citep{mathew2021docvqa}, TextVQA \citep{singh2019towards}, LLaVAR \citep{zhang2023llavar}, ST-VQA \citep{biten2019scene} \\
\hline
\end{tabular}
\arrayrulecolor{black}
\vspace{-1.9em}
\end{table*}

\paragraph{Training configuration.}
All experts are trained with LoRA applied only to non-visual modules, while all visual-side parameters remain frozen, using $(r,\alpha_{\mathrm{LoRA}})=(32,64)$ for Qwen2.5-VL-3B and $(r,\alpha_{\mathrm{LoRA}})=(16,32)$ for Qwen2.5-VL-7B and InternVL3.5-8B. The learning rate is set to $2\times10^{-5}$ for Qwen2.5-VL-3B and $1\times10^{-5}$ for the two larger configurations, with a cosine learning-rate schedule. Each expert is trained for one epoch to limit deviation from the shared initialization and preserve compatibility for subsequent model merging.

\subsection{Evaluation Benchmarks}
\label{app:evaluation-benchmarks}

\paragraph{Expert-task benchmarks.} \textbf{Chart understanding.} ChartQA \citep{masry2022chartqa} evaluates question answering over charts, requiring models to extract visualized information and perform logical or arithmetic reasoning over chart content. \textbf{Geometry reasoning.} MATH-Vision \citep{wang2024measuring} contains competition-sourced visual mathematics problems that require joint reasoning over visual contexts and textual problem statements, while Geometry3K \citep{lu2021inter} focuses on solving geometry problems by integrating textual conditions with diagram structure and geometric relations. \textbf{OCR.} TextVQA \citep{singh2019towards} evaluates the ability to read and reason about text embedded in natural images, whereas DocVQA \citep{mathew2021docvqa} requires question answering over document images using their textual content and document structure. \textbf{Visual question answering.} GQA \citep{hudson2019gqa} emphasizes compositional reasoning over real-world scenes, while OKVQA \citep{marino2019ok} contains questions that require knowledge beyond the information directly visible in the image. We report the unweighted mean over these seven benchmarks as \textit{Expert Avg.}

\paragraph{Hallucination benchmarks.}
\textbf{AMBER.} AMBER \citep{wang2023amber} evaluates generative and discriminative hallucinations; we use its generative image-description setting and report CHAIR, HalRate, and Cog.
\textbf{Object HalBench.} Object HalBench \citep{rohrbach2018object} measures object hallucination in image descriptions and reports response-level $\mathrm{CHAIR}_{s}$ and mention-level $\mathrm{CHAIR}_{i}$.
\textbf{POPE.} POPE \citep{li2023evaluating} evaluates object hallucination through binary object-existence questions under Random, Popular, and Adversarial splits; we report accuracy and F1.
\textbf{GPT-4-assisted evaluation.} Following \citet{zhao2023beyond}, we use Visual Genome \citep{krishna2017visual} annotations and GPT-4o-mini to assess fine-grained hallucinations in generated descriptions, reporting Sentence-level Hallucination Ratio (SHR).

%% file: secs/A_additional_results.tex
\section{Additional Results}
\label{app:additional-results}

\paragraph{Complete InternVL3.5-8B results.}
\label{app:internvl-results}

Tables~\ref{tab:internvl-expert} and~\ref{tab:internvl-hallucination} provide the complete matched results for InternVL3.5-8B, complementing the Qwen2.5-VL results in the main text. Across the eight model-merging methods, MergeHEIR consistently improves hallucination performance while largely maintaining expert-task performance. This pattern is consistent with the intended role of MergeHEIR in coordinating hallucination correction with the retention of inherited expertise.

\paragraph{Complete POPE results.}
\label{app:complete-pope}

\begin{table*}[t]
\caption{Expert-task results on InternVL3.5-8B.}
\label{tab:internvl-expert}
\centering
\scriptsize
\setlength{\tabcolsep}{1.0pt}
\begin{tabular}{@{}C{70pt}C{38pt}|C{27pt}|cc|C{27pt}C{27pt}|C{27pt}C{27pt}|C{27pt}@{}}
\toprule
\multirow[c]{2}{*}{Merging method} & \multirow[c]{2}{*}{Stage} & \multicolumn{1}{c|}{Chart} & \multicolumn{2}{c|}{Geometry} & \multicolumn{2}{c|}{OCR} & \multicolumn{2}{c|}{VQA} & \multirow[c]{2}{*}{Avg.}\\
 & & ChartQA & MATH-Vision & Geometry3K & TextVQA & DocVQA & GQA & OKVQA\\
\midrule
Task Arithmetic & Merged & 76.20 & 23.68 & 60.23 & 63.83 & 56.46 & 59.96 & 42.17 & 54.65\\
\rowcolor{mergeheirrow} & MergeHEIR & 76.08 & 23.03 & 60.57 & 63.62 & 56.26 & 60.84 & 39.61 & 54.29\\
TIES Merging & Merged & 76.52 & 24.67 & 60.07 & 64.21 & 56.85 & 60.93 & 43.87 & 55.30\\
\rowcolor{mergeheirrow} & MergeHEIR & 76.76 & 23.03 & 59.73 & 64.50 & 57.29 & 61.78 & 42.71 & 55.11\\
DELLA Merging & Merged & 77.08 & 24.67 & 60.07 & 64.09 & 56.63 & 60.57 & 43.39 & 55.21\\
\rowcolor{mergeheirrow} & MergeHEIR & 77.12 & 23.68 & 59.73 & 63.97 & 56.54 & 61.50 & 42.04 & 54.94\\
Breadcrumbs Merging & Merged & 76.80 & 24.34 & 60.07 & 64.14 & 56.69 & 60.41 & 43.12 & 55.08\\
\rowcolor{mergeheirrow} & MergeHEIR & 76.92 & 24.34 & 60.73 & 64.07 & 56.77 & 61.40 & 40.44 & 54.95\\
Model Stock & Merged & 75.72 & 23.68 & 59.90 & 63.76 & 55.90 & 59.49 & 40.67 & 54.16\\
\rowcolor{mergeheirrow} & MergeHEIR & 75.52 & 22.04 & 59.23 & 63.62 & 56.19 & 59.76 & 39.52 & 53.70\\
Iso-C & Merged & 75.76 & 24.01 & 60.07 & 63.61 & 56.13 & 59.53 & 40.83 & 54.28\\
\rowcolor{mergeheirrow} & MergeHEIR & 75.68 & 22.04 & 59.23 & 63.49 & 56.06 & 59.79 & 39.07 & 53.63\\
WUDI Merging & Merged & 77.36 & 25.33 & 60.57 & 64.45 & 56.89 & 61.11 & 44.42 & 55.73\\
\rowcolor{mergeheirrow} & MergeHEIR & 77.28 & 23.68 & 60.07 & 64.40 & 57.49 & 61.63 & 43.77 & 55.47\\
OptMerge & Merged & 77.68 & 25.33 & 60.57 & 64.29 & 57.09 & 61.64 & 46.33 & 56.13\\
\rowcolor{mergeheirrow} & MergeHEIR & 77.60 & 25.00 & 61.23 & 64.30 & 57.04 & 62.78 & 44.11 & 56.01\\
\bottomrule
\end{tabular}
\end{table*}

\begin{table*}[t]
\vspace{-0.8em}
\caption{Hallucination-task results on InternVL3.5-8B.}
\label{tab:internvl-hallucination}
\centering
\scriptsize
\setlength{\tabcolsep}{1.5pt}
\begin{tabular}{@{}C{70pt}C{38pt}|ccc|cc|C{26pt}C{26pt}|c@{}}
\toprule
\multirow[c]{2}{*}{Merging method} & \multirow[c]{2}{*}{Stage} & \multicolumn{3}{c|}{AMBER} & \multicolumn{2}{c|}{Object HalBench} & \multicolumn{2}{c|}{POPE Adversarial} & \multicolumn{1}{c}{GPT-4 Eval}\\
 & & CHAIR$\downarrow$ & HalRate$\downarrow$ & Cog$\downarrow$ & CHAIR$_s\downarrow$ & CHAIR$_i\downarrow$ & Acc.$\uparrow$ & F1$\uparrow$ & SHR$\downarrow$\\
\midrule
Task Arithmetic & Merged & 8.3 & 66.5 & 7.3 & 25.3 & 6.7 & 73.9& 78.6 & 36.2\\
\rowcolor{mergeheirrow} & MergeHEIR & 7.6 & 63.1 & 6.3 & 24.3 & 6.7 & 77.5& 80.6 & 32.8\\
TIES Merging & Merged & 8.1 & 68.7 & 7.1 & 24.0 & 6.2 & 76.5& 80.0 & 36.3\\
\rowcolor{mergeheirrow} & MergeHEIR & 7.2 & 62.3 & 6.1 & 20.3 & 5.3 & 80.6 & 82.5 & 33.6\\
DELLA Merging & Merged & 8.1 & 68.7 & 7.3 & 23.7 & 6.2 & 75.3& 79.4 & 36.9\\
\rowcolor{mergeheirrow} & MergeHEIR & 7.3 & 63.4 & 6.4 & 20.7 & 5.2 & 79.1 & 81.5 & 33.1\\
Breadcrumbs Merging & Merged & 8.2 & 69.9 & 7.1 & 22.7 & 5.9 & 75.1& 79.2 & 36.2\\
\rowcolor{mergeheirrow} & MergeHEIR & 6.9 & 63.0 & 6.3 & 20.3 & 5.1 & 79.2 & 81.5 & 32.4\\
Model Stock & Merged & 8.0 & 65.6 & 6.8 & 27.7 & 6.8 & 73.3& 78.2 & 36.7\\
\rowcolor{mergeheirrow} & MergeHEIR & 7.7 & 63.3 & 6.2 & 23.3 & 6.4 & 77.2& 80.5 & 34.0\\
Iso-C & Merged & 8.4 & 66.6 & 6.9 & 27.7 & 6.8 & 73.5& 78.3 & 36.7\\
\rowcolor{mergeheirrow} & MergeHEIR & 7.8 & 62.8 & 6.0 & 23.7 & 6.5 & 77.1& 80.3 & 33.8\\
WUDI Merging & Merged & 8.0 & 68.4 & 6.9 & 24.0 & 5.9 & 78.2& 80.9 & 35.9\\
\rowcolor{mergeheirrow} & MergeHEIR & 7.1 & 63.6 & 6.1 & 21.7 & 5.3 & 82.4 & 83.3 & 33.2\\
OptMerge & Merged & 7.6 & 65.0 & 6.8 & 18.0 & 4.2 & 80.8& 82.6 & 36.5\\
\rowcolor{mergeheirrow} & MergeHEIR & 6.8 & 61.2 & 6.0 & 17.3 & 4.8 & 86.6 & 88.2 & 32.4\\
\bottomrule
\end{tabular}
\end{table*}

\begin{table*}[t]
\caption{Complete POPE results for Qwen2.5-VL-3B.}
\label{tab:appendix-pope-Qwen25-VL-3B}
\centering
\scriptsize
\setlength{\tabcolsep}{2.2pt}
\begin{tabular}{@{}C{70pt}C{38pt}|C{30pt}C{30pt}|C{30pt}C{30pt}|C{30pt}C{30pt}@{}}
\toprule
\multirow[c]{2}{*}{Merging method} & \multirow[c]{2}{*}{Stage} & \multicolumn{2}{c|}{POPE Adversarial} & \multicolumn{2}{c|}{POPE Popular} & \multicolumn{2}{c}{POPE Random}\\
 & & Acc.$\uparrow$ & F1$\uparrow$ & Acc.$\uparrow$ & F1$\uparrow$ & Acc.$\uparrow$ & F1$\uparrow$\\
\midrule
Task Arithmetic & Merged & 82.17 & 83.47 & 85.83 & 86.41 & 91.96 & 92.03\\
\rowcolor{mergeheirrow} & MergeHEIR & 85.80 & 86.05 & 88.63 & 88.51 & 92.23 & 92.08\\
TIES Merging & Merged & 82.17 & 83.46 & 85.77 & 86.34 & 91.89 & 91.96\\
\rowcolor{mergeheirrow} & MergeHEIR & 85.80 & 86.06 & 88.70 & 88.58 & 92.20 & 92.05\\
DELLA Merging & Merged & 82.07 & 83.37 & 85.97 & 86.50 & 91.89 & 91.96\\
\rowcolor{mergeheirrow} & MergeHEIR & 85.77 & 86.04 & 88.73 & 88.62 & 92.23 & 92.09\\
Breadcrumbs Merging & Merged & 81.90 & 83.21 & 85.70 & 86.25 & 91.82 & 91.88\\
\rowcolor{mergeheirrow} & MergeHEIR & 85.70 & 85.99 & 88.67 & 88.57 & 92.30 & 92.16\\
Model Stock & Merged & 82.87 & 83.91 & 87.17 & 87.44 & 91.96 & 91.97\\
\rowcolor{mergeheirrow} & MergeHEIR & 86.60 & 86.45 & 89.93 & 89.52 & 92.20 & 92.02\\
Iso-C & Merged & 82.77 & 83.87 & 87.23 & 87.53 & 91.99 & 92.02\\
\rowcolor{mergeheirrow} & MergeHEIR & 85.73 & 86.05 & 88.70 & 88.62 & 92.30 & 92.18\\
WUDI Merging & Merged & 85.37 & 85.24 & 88.00 & 87.57 & 90.82 & 90.47\\
\rowcolor{mergeheirrow} & MergeHEIR & 89.13 & 87.92 & 90.63 & 89.67 & 91.41 & 90.99\\
OptMerge & Merged & 83.53 & 84.16 & 86.80 & 86.89 & 91.75 & 91.62\\
\rowcolor{mergeheirrow} & MergeHEIR & 87.33 & 86.95 & 89.73 & 89.17 & 92.13 & 91.78\\
\bottomrule
\end{tabular}
\end{table*}

\begin{table*}[t]
\caption{Complete POPE results for Qwen2.5-VL-7B.}
\label{tab:appendix-pope-Qwen25-VL-7B}
\centering
\scriptsize
\setlength{\tabcolsep}{2.2pt}
\begin{tabular}{@{}C{70pt}C{38pt}|C{30pt}C{30pt}|C{30pt}C{30pt}|C{30pt}C{30pt}@{}}
\toprule
\multirow[c]{2}{*}{Merging method} & \multirow[c]{2}{*}{Stage} & \multicolumn{2}{c|}{POPE Adversarial} & \multicolumn{2}{c|}{POPE Popular} & \multicolumn{2}{c}{POPE Random}\\
 & & Acc.$\uparrow$ & F1$\uparrow$ & Acc.$\uparrow$ & F1$\uparrow$ & Acc.$\uparrow$ & F1$\uparrow$\\
\midrule
Task Arithmetic & Merged & 83.37 & 80.70 & 83.87 & 81.17 & 84.09 & 81.84\\
\rowcolor{mergeheirrow} & MergeHEIR & 86.97 & 86.00 & 88.20 & 87.09 & 88.93 & 88.14\\
TIES Merging & Merged & 85.53 & 84.41 & 87.27 & 86.02 & 87.84 & 86.91\\
\rowcolor{mergeheirrow} & MergeHEIR & 89.17 & 88.05 & 91.07 & 90.19 & 92.34 & 92.06\\
DELLA Merging & Merged & 84.00 & 81.85 & 84.67 & 82.47 & 85.26 & 83.46\\
\rowcolor{mergeheirrow} & MergeHEIR & 86.97 & 86.50 & 88.33 & 87.68 & 89.86 & 89.31\\
Breadcrumbs Merging & Merged & 83.93 & 81.69 & 84.57 & 82.28 & 85.12 & 83.24\\
\rowcolor{mergeheirrow} & MergeHEIR & 87.47 & 86.87 & 88.70 & 88.02 & 89.66 & 89.45\\
Model Stock & Merged & 83.20 & 80.36 & 83.63 & 80.77 & 83.78 & 81.37\\
\rowcolor{mergeheirrow} & MergeHEIR & 86.80 & 85.65 & 87.77 & 86.56 & 88.52 & 87.62\\
Iso-C & Merged & 83.37 & 80.53 & 83.73 & 80.88 & 83.81 & 81.42\\
\rowcolor{mergeheirrow} & MergeHEIR & 86.70 & 85.54 & 87.70 & 86.48 & 88.52 & 87.60\\
WUDI Merging & Merged & 86.00 & 85.10 & 87.90 & 86.85 & 88.63 & 87.87\\
\rowcolor{mergeheirrow} & MergeHEIR & 89.43 & 88.94 & 91.77 & 91.12 & 93.16 & 92.94\\
OptMerge & Merged & 86.10 & 84.92 & 87.30 & 86.04 & 88.01 & 87.06\\
\rowcolor{mergeheirrow} & MergeHEIR & 89.70 & 89.29 & 91.40 & 90.73 & 92.61 & 92.33\\
\bottomrule
\end{tabular}
\end{table*}

\begin{table*}[t]
\caption{Complete POPE results for InternVL3.5-8B.}
\label{tab:appendix-pope-InternVL35-8B}
\centering
\scriptsize
\setlength{\tabcolsep}{2.2pt}
\begin{tabular}{@{}C{70pt}C{38pt}|C{30pt}C{30pt}|C{30pt}C{30pt}|C{30pt}C{30pt}@{}}
\toprule
\multirow[c]{2}{*}{Merging method} & \multirow[c]{2}{*}{Stage} & \multicolumn{2}{c|}{POPE Adversarial} & \multicolumn{2}{c|}{POPE Popular} & \multicolumn{2}{c}{POPE Random}\\
 & & Acc.$\uparrow$ & F1$\uparrow$ & Acc.$\uparrow$ & F1$\uparrow$ & Acc.$\uparrow$ & F1$\uparrow$\\
\midrule
Task Arithmetic & Merged & 73.93 & 78.61 & 78.43 & 81.62 & 92.20 & 92.68\\
\rowcolor{mergeheirrow} & MergeHEIR & 77.50 & 80.64 & 82.20 & 84.04 & 93.40 & 93.61\\
TIES Merging & Merged & 76.53 & 80.02 & 81.50 & 83.56 & 92.85 & 93.13\\
\rowcolor{mergeheirrow} & MergeHEIR & 80.63 & 82.52 & 85.73 & 86.35 & 94.60 & 94.48\\
DELLA Merging & Merged & 75.33 & 79.36 & 79.73 & 82.40 & 92.41 & 92.79\\
\rowcolor{mergeheirrow} & MergeHEIR & 79.13 & 81.51 & 83.77 & 85.02 & 93.81 & 93.90\\
Breadcrumbs Merging & Merged & 75.07 & 79.21 & 79.40 & 82.18 & 92.34 & 92.74\\
\rowcolor{mergeheirrow} & MergeHEIR & 79.17 & 81.49 & 83.83 & 84.96 & 94.02 & 94.01\\
Model Stock & Merged & 73.30 & 78.23 & 77.43 & 80.96 & 92.03 & 92.54\\
\rowcolor{mergeheirrow} & MergeHEIR & 77.23 & 80.46 & 81.70 & 83.67 & 93.47 & 93.67\\
Iso-C & Merged & 73.47 & 78.33 & 77.50 & 81.00 & 92.03 & 92.54\\
\rowcolor{mergeheirrow} & MergeHEIR & 77.07 & 80.35 & 81.73 & 83.70 & 93.40 & 93.61\\
WUDI Merging & Merged & 78.23 & 80.93 & 83.13 & 84.56 & 92.78 & 92.96\\
\rowcolor{mergeheirrow} & MergeHEIR & 82.37 & 83.31 & 87.60 & 87.44 & 94.47 & 94.37\\
OptMerge & Merged & 80.80 & 82.61 & 85.07 & 85.93 & 93.02 & 93.09\\
\rowcolor{mergeheirrow} & MergeHEIR & 86.60 & 88.18 & 90.67 & 91.46 & 94.60 & 94.87\\
\bottomrule
\end{tabular}
\end{table*}

Tables~\ref{tab:appendix-pope-Qwen25-VL-3B}--\ref{tab:appendix-pope-InternVL35-8B} report accuracy and F1 on all three POPE splits. Across all 24 cases, MergeHEIR improves both metrics on every split.

\begin{table*}[t]
\caption{Compatibility of MergeHEIR with alternative hallucination-oriented preference methods. ``+ HEIR'' denotes augmenting the corresponding method with the null-space projection mechanism of MergeHEIR.}
\label{tab:plugin-compatibility}
\centering
\scriptsize
\setlength{\tabcolsep}{2.8pt}
\begin{tabular}{cc|C{38pt}|ccC{22pt}|cc}
\toprule
\multirow{2}{*}{Model} &
\multirow{2}{*}{Setting} &
\multirow{2}{*}{\shortstack{Expert\\Avg.$\uparrow$}} &
\multicolumn{3}{c|}{AMBER} &
\multicolumn{2}{c}{Object HalBench} \\
& & &
CHAIR$\downarrow$ &
HalRate$\downarrow$ &
Cog$\downarrow$ &
CHAIR$_s\downarrow$ &
CHAIR$_i\downarrow$ \\
\midrule

\multirow{5}{*}{Qwen2.5-VL-3B}
& Merged & 60.83 & 8.8 & 59.6 & 7.1 & 36.0 & 9.0 \\
& RLAIF-V & 60.28 & 5.3 & 32.2 & 3.5 & 27.7 & 7.9 \\
& RLAIF-V + HEIR & 60.86 & 5.5 & 30.8 & 3.6 & 28.0 & 7.7 \\
& OPA-DPO & 60.15 & 4.5 & 21.2 & 1.8 & 21.3 & 6.1 \\
& OPA-DPO + HEIR & 60.90 & 4.6 & 21.5 & 2.0 & 21.7 & 6.2 \\
\midrule

\multirow{5}{*}{Qwen2.5-VL-7B}
& Merged & 65.30 & 4.9 & 24.3 & 1.4 & 41.0 & 9.0 \\
& RLAIF-V & 63.61 & 4.6 & 22.4 & 1.3 & 30.3 & 7.5 \\
& RLAIF-V + HEIR & 65.23 & 4.5 & 22.5 & 1.3 & 30.7 & 7.8 \\
& OPA-DPO & 63.32 & 4.1 & 19.6 & 1.0 & 24.0 & 6.3 \\
& OPA-DPO + HEIR & 65.18 & 4.2 & 20.0 & 1.0 & 24.7 & 6.1 \\
\midrule

\multirow{5}{*}{InternVL3.5-8B}
& Merged & 55.73 & 8.0 & 68.4 & 6.9 & 24.0 & 5.9 \\
& RLAIF-V & 54.30 & 7.6 & 65.8 & 6.8 & 23.0 & 5.6 \\
& RLAIF-V + HEIR & 55.58 & 7.6 & 66.2 & 6.6 & 22.7 & 5.7 \\
& OPA-DPO & 54.21 & 7.0 & 62.4 & 6.0 & 20.7 & 5.2 \\
& OPA-DPO + HEIR & 55.52 & 7.2 & 63.2 & 6.1 & 21.0 & 5.1 \\
\bottomrule
\end{tabular}
\end{table*}

\paragraph{Compatibility with alternative preference methods.}
To examine whether the retention benefit of MergeHEIR depends on its specific preference objective, we integrate its null-space projection mechanism with RLAIF-V \citep{yu2025rlaif} and OPA-DPO \citep{yang2025mitigating}. Across all six matched comparisons, direct adaptation with the underlying preference methods causes substantially greater degradation of inherited expert capabilities, whereas incorporating MergeHEIR's null-space projection preserves expert performance while maintaining essentially the same hallucination-mitigation benefit. These results indicate that the retention benefit of MergeHEIR is not tied to its specific preference objective and support its use as a plug-and-play mechanism for capability-preserving post-merge adaptation.

%% file: secs/A_case_study.tex
\section{Case Study}
\label{sec:case-study}
We provide qualitative comparisons among the merged model, MergeHEIR, and unconstrained post-merge adaptation on both expert tasks and hallucination benchmarks. As shown in Figs.~\ref{fig:case-chart}--\ref{fig:case-vqa}, unconstrained adaptation can substantially impair the capabilities inherited from the constituent experts, producing incorrect chart interpretations, flawed geometric reasoning, inaccurate text recognition, or erroneous visual answers. In contrast, MergeHEIR retains the task-specific knowledge of the merged model and produces responses that remain consistent with the visual evidence.

Figs.~\ref{fig:case-pope} and~\ref{fig:case-objhalbench} further compare the three model variants on hallucination evaluation examples. The merged model tends to introduce nonexistent objects or visually unsupported scene details, whereas both post-merge adaptation methods effectively suppress such hallucinations. Notably, MergeHEIR achieves this behavioral correction while avoiding the pronounced expertise degradation exhibited by unconstrained adaptation. These examples qualitatively corroborate the quantitative results, demonstrating that MergeHEIR achieves a more favorable trade-off between hallucination mitigation and expertise retention.

\begin{figure*}[p]
    \centering
    \includegraphics[
        page=1,
        width=\textwidth,
        height=0.92\textheight,
        keepaspectratio
    ]{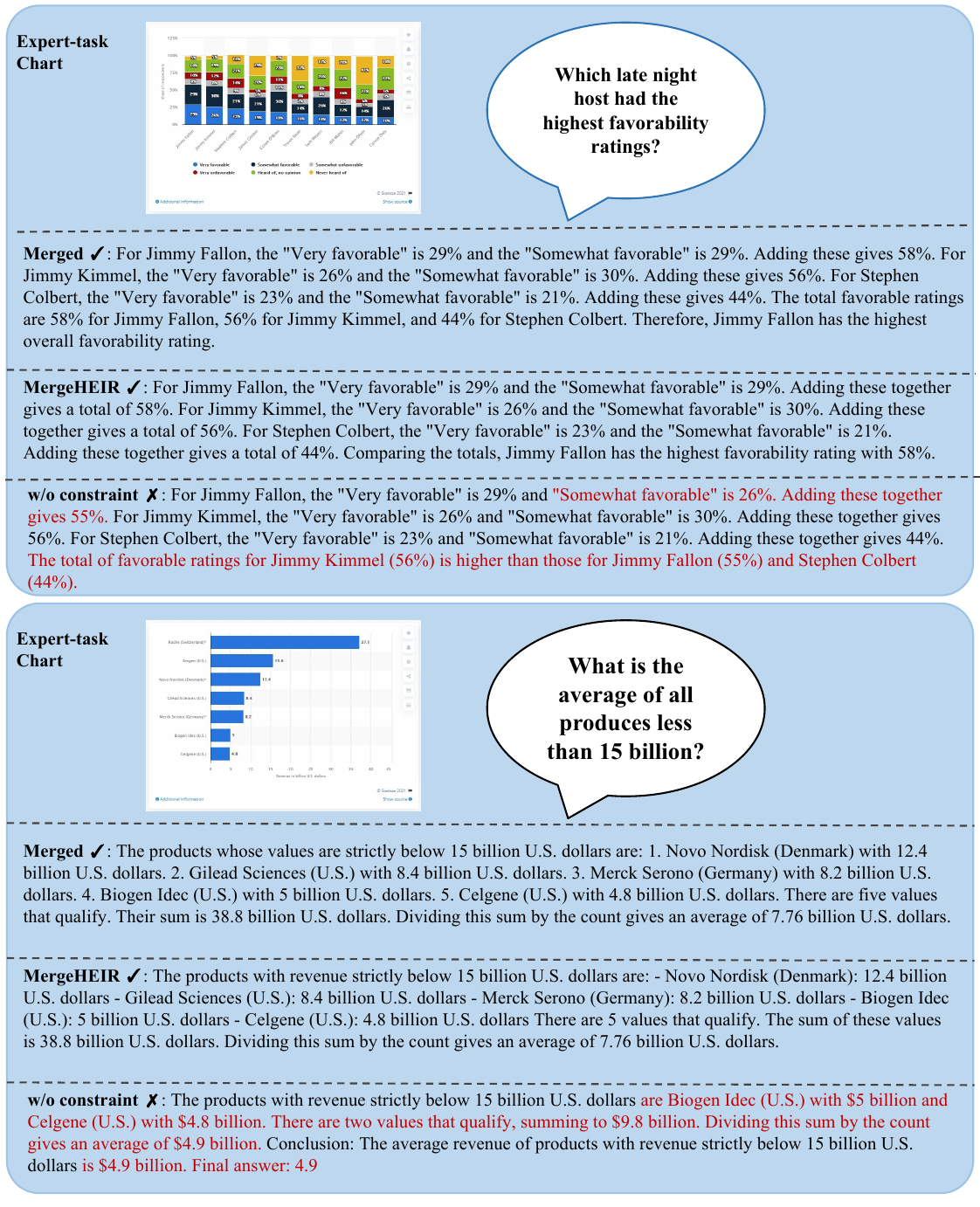}
    \caption{
        Visualization results comparing the merged model, MergeHEIR, and unconstrained post-merge adaptation on chart understanding task. Incorrect or visually unsupported response spans are highlighted in \textcolor{red}{red}.
    }
    \label{fig:case-chart}
\end{figure*}

\begin{figure*}[p]
    \centering
    \includegraphics[
        page=2,
        width=\textwidth,
        height=0.92\textheight,
        keepaspectratio
    ]{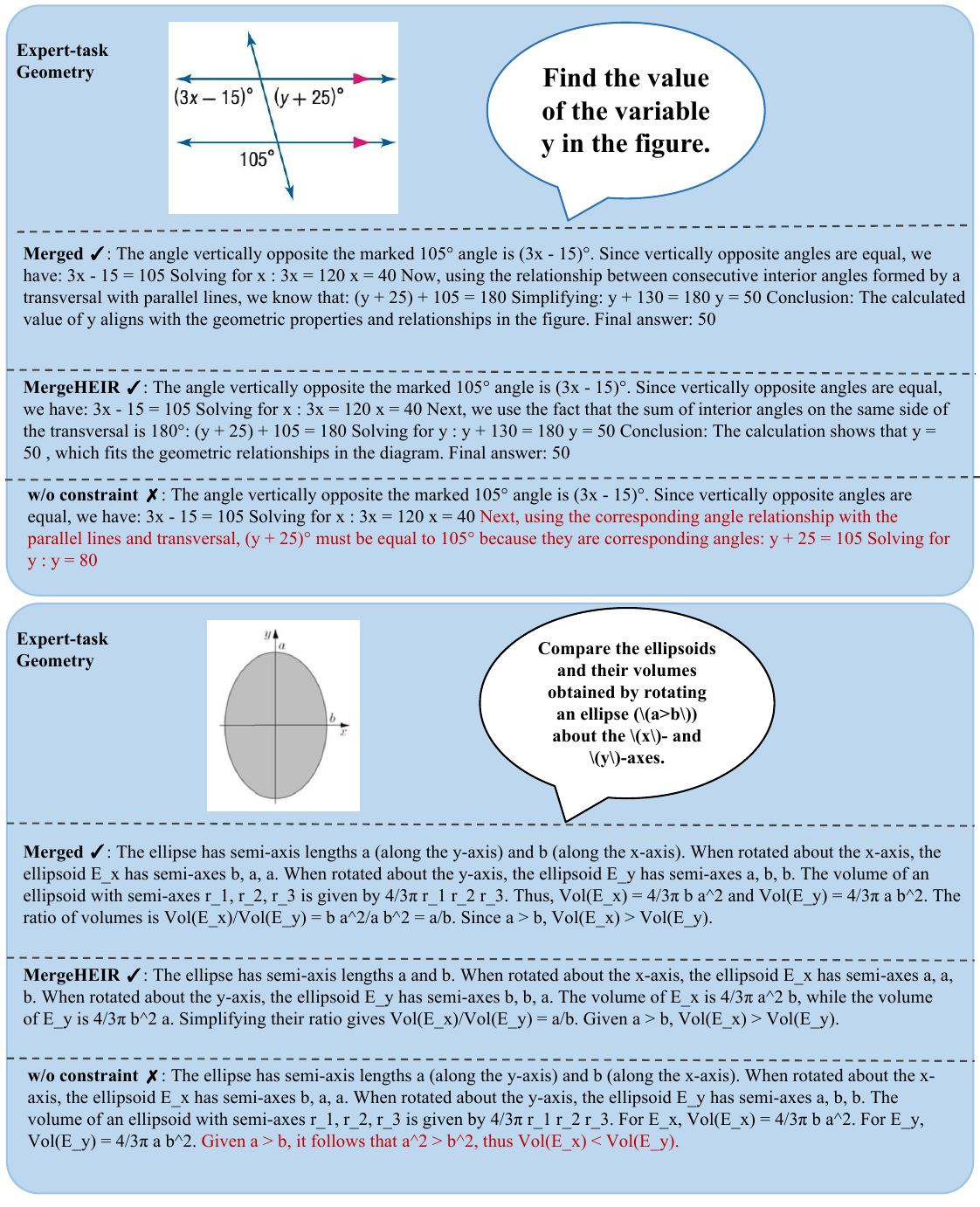}
    \caption{
        Visualization results comparing the merged model, MergeHEIR, and unconstrained post-merge adaptation on geometry reasoning task. Incorrect or visually unsupported response spans are highlighted in \textcolor{red}{red}.
    }
    \label{fig:case-geometry}
\end{figure*}

\begin{figure*}[p]
    \centering
    \includegraphics[
        page=3,
        width=\textwidth,
        height=0.92\textheight,
        keepaspectratio
    ]{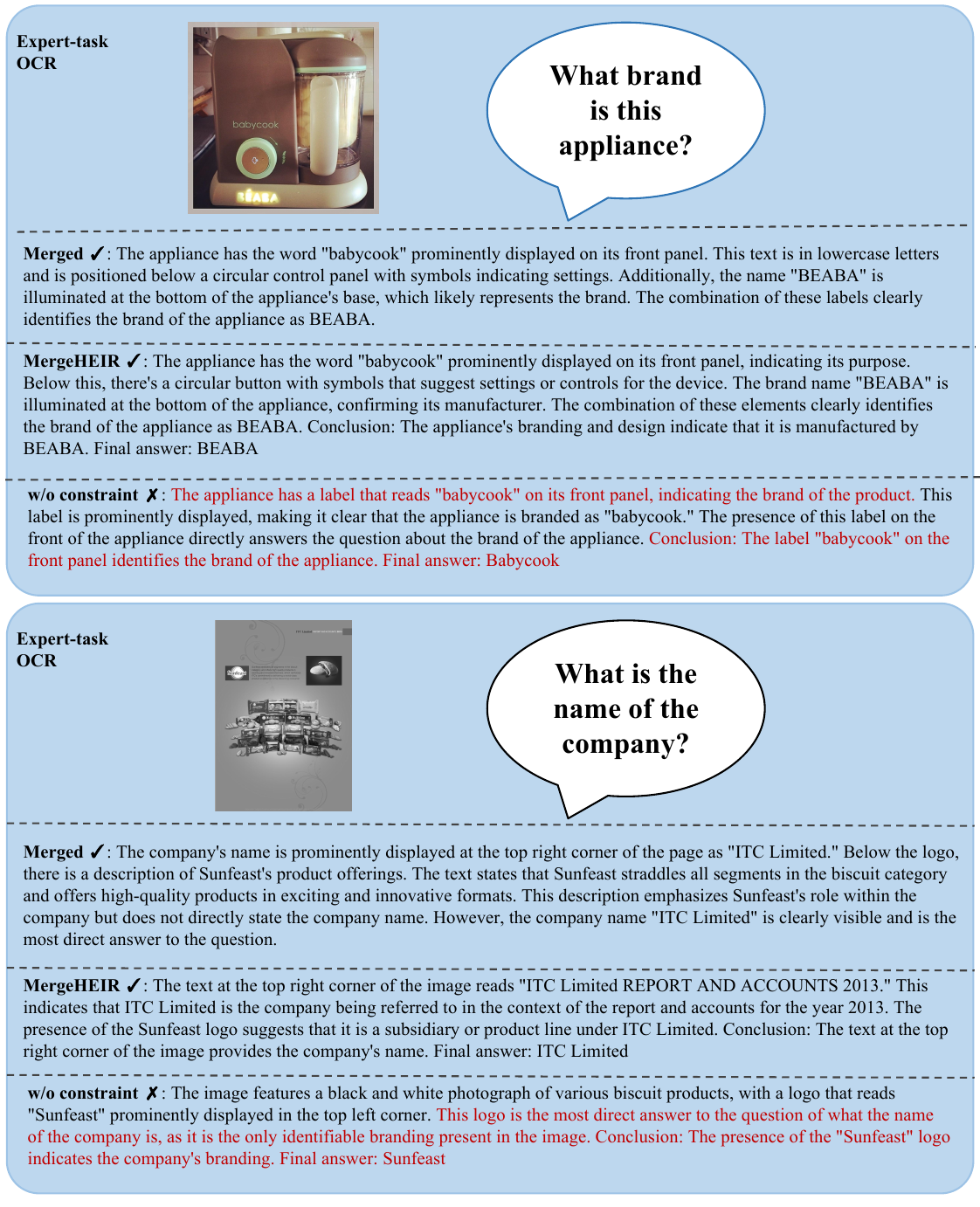}
    \caption{
        Visualization results comparing the merged model, MergeHEIR, and unconstrained post-merge adaptation on OCR task. Incorrect or visually unsupported response spans are highlighted in \textcolor{red}{red}.
    }
    \label{fig:case-ocr}
\end{figure*}

\begin{figure*}[p]
    \centering
    \includegraphics[
        page=4,
        width=\textwidth,
        height=0.92\textheight,
        keepaspectratio
    ]{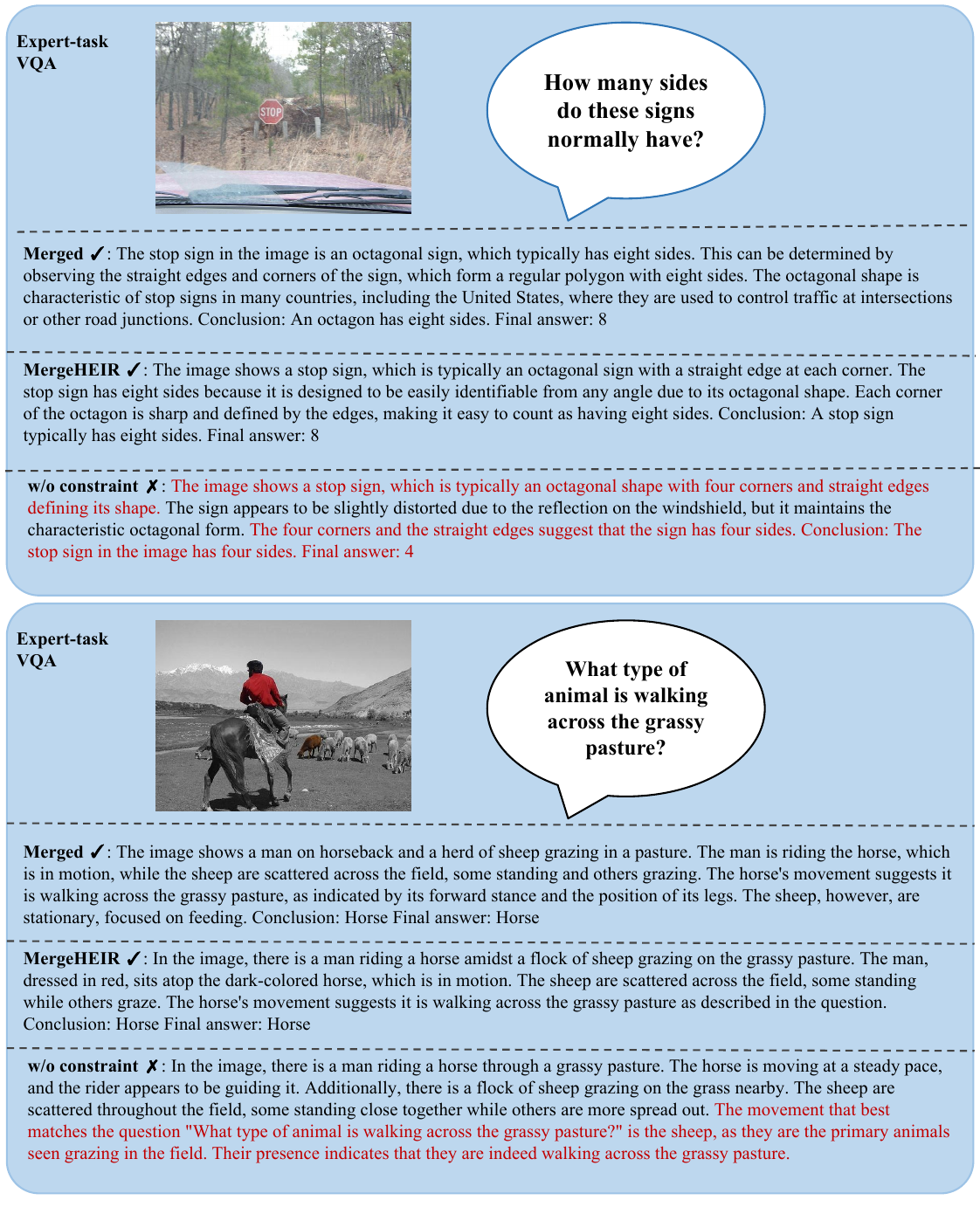}
    \caption{
    Visualization results comparing the merged model, MergeHEIR, and unconstrained post-merge adaptation on VQA task. Incorrect or visually unsupported response spans are highlighted in \textcolor{red}{red}.
    }
    \label{fig:case-vqa}
\end{figure*}

\begin{figure*}[p]
    \centering
    \includegraphics[
        page=5,
        width=\textwidth,
        height=0.92\textheight,
        keepaspectratio
    ]{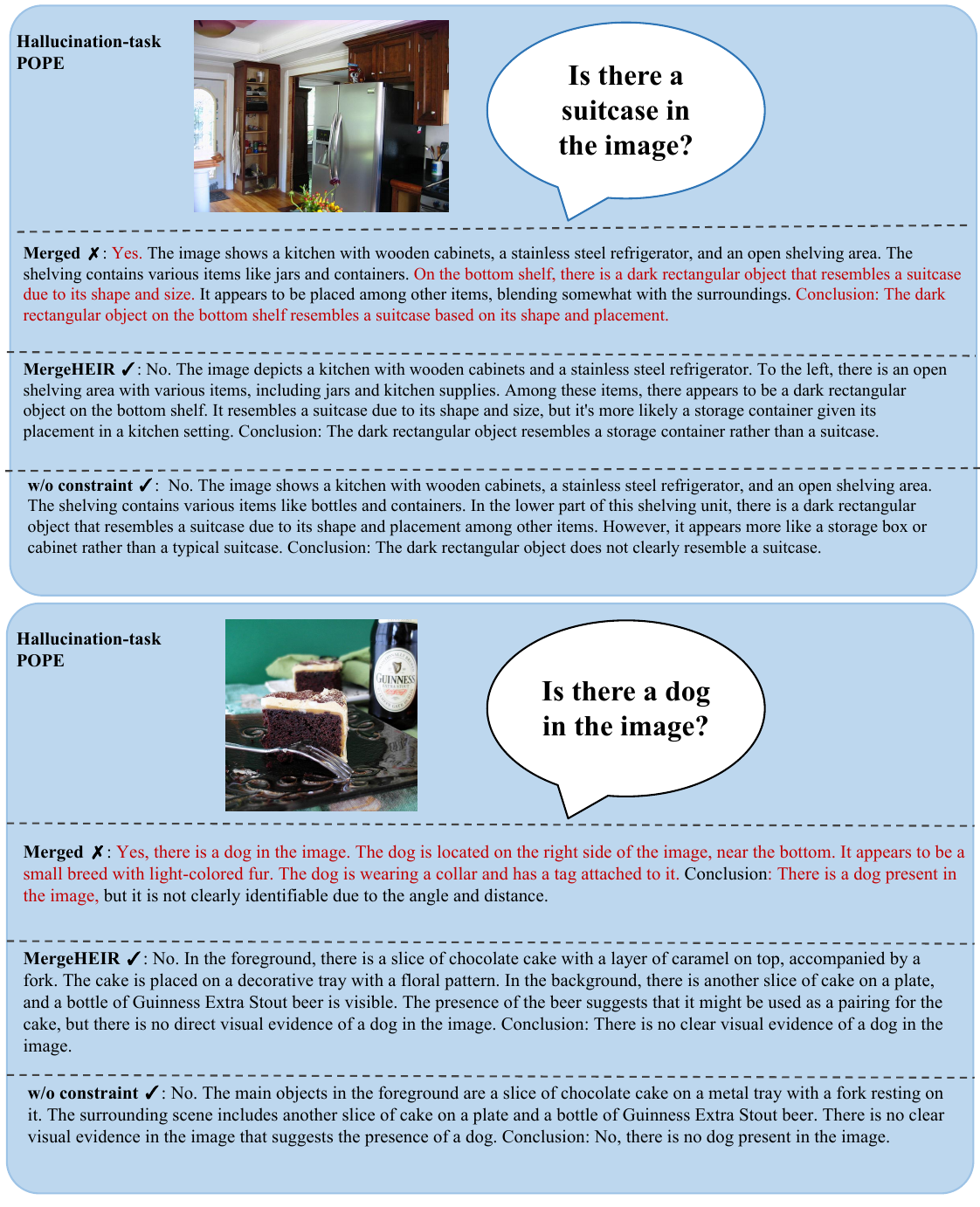}
    \caption{
    Visualization results comparing the merged model, MergeHEIR, and unconstrained post-merge adaptation on POPE. Hallucinated or visually unsupported content is highlighted in \textcolor{red}{red}.
    }
    \label{fig:case-pope}
\end{figure*}

\begin{figure*}[p]
    \centering
    \includegraphics[
        page=6,
        width=\textwidth,
        height=0.92\textheight,
        keepaspectratio
    ]{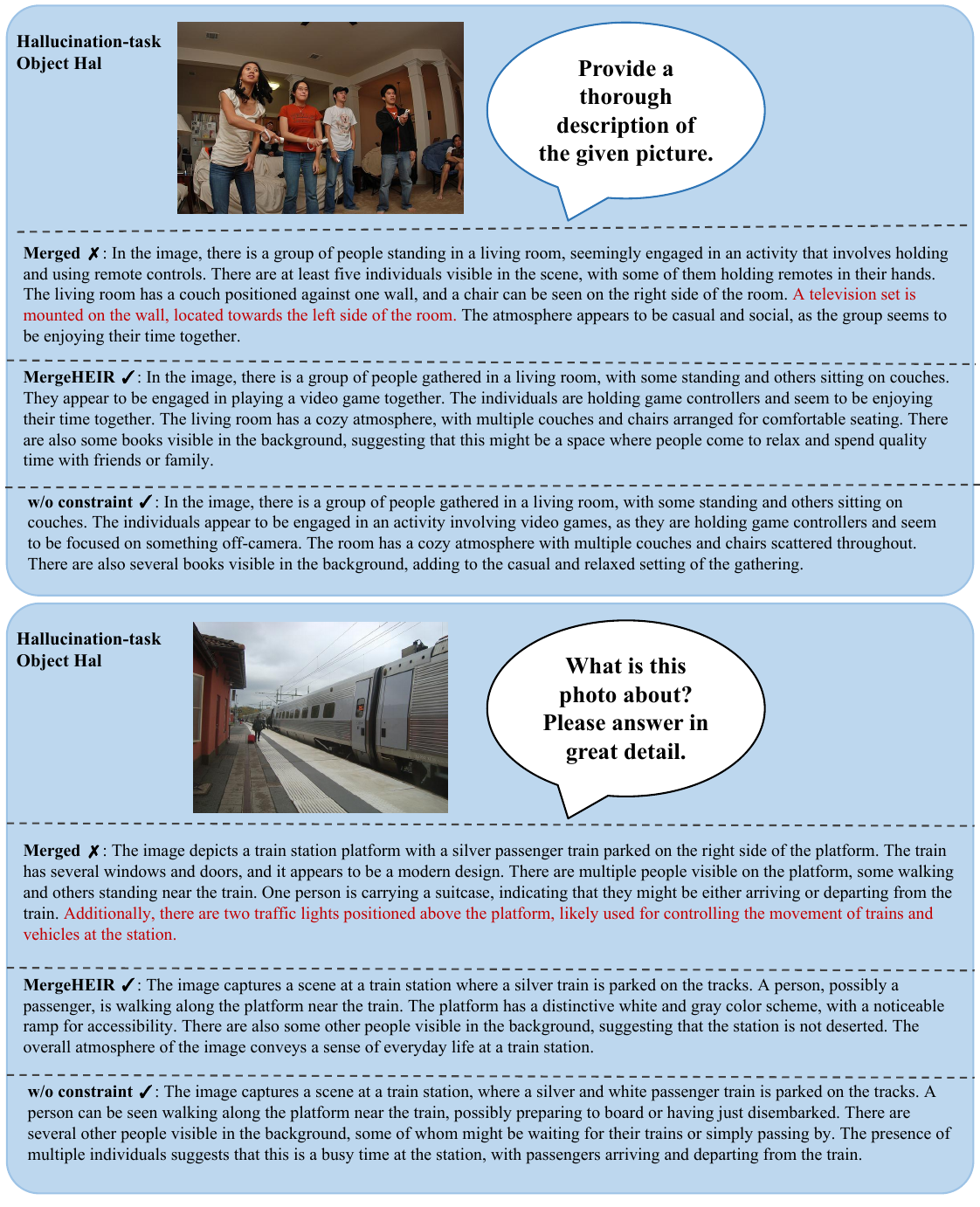}
    \caption{
      Visualization results comparing the merged model, MergeHEIR, and unconstrained post-merge adaptation on ObjHalBench. Hallucinated objects and visually unsupported scene details are highlighted in \textcolor{red}{red}.
    }
    \label{fig:case-objhalbench}
\end{figure*}

%% file: secs/X_statements.tex
\section{Impact Statement}
This paper takes an early step toward mitigating hallucinations in merged MLLMs while preserving inherited expert capabilities. The goal is to improve the reliability of post-merge adaptation and support more trustworthy multimodal systems. The broader societal impacts are largely consistent with those associated with improving the robustness and reliability of machine learning systems, and we do not identify additional consequences that require specific discussion.

%% file: secs/A_llm_use.tex
\section{Statement on LLM assistance}
An LLM assistant was used for writing refinement and formatting adjustments.